\PassOptionsToPackage{usenames,svgnames}{xcolor}

\documentclass[a4, 11pt, parskip=half, DIV=10]{scrartcl}
\usepackage[english]{babel}
\usepackage[T1]{fontenc}
\usepackage[utf8]{inputenc}

\usepackage{graphicx}
\usepackage[numbers]{natbib}
\usepackage{subcaption}
\usepackage{booktabs}
\usepackage{url}
\usepackage{todonotes}
\usepackage{multirow}
\usepackage{csquotes}
\usepackage{paralist}

\usepackage{pgfplots}
\usepgfplotslibrary{colorbrewer}

\pgfdeclarelayer{background}
\pgfsetlayers{background,main}

\usepackage{bm}

\usepackage{amsmath}
\usepackage{amssymb}
\usepackage{amsthm}
\usepackage{amsfonts}
\usepackage{thmtools}		
\usepackage{mleftright}
\usepackage{stmaryrd}
\usepackage{nicefrac}
\usepackage{abstract}

\theoremstyle{definition}
\newtheorem{theorem}{Theorem}
\newtheorem{proposition}[theorem]{Proposition}

\newtheorem{definition}[theorem]{Definition}

\usepackage{thm-restate}
\usepackage[mathic=true]{mathtools}
\usepackage{fixmath}
\usepackage{siunitx}
\usepackage{color}

\usepackage{pifont}

\usepackage{enumitem}
\setlist[enumerate]{itemsep=0.2ex, topsep=0.5\topsep}
\setlist[description]{itemsep=0.2ex, topsep=0.5\topsep}
\setlist[itemize]{itemsep=0.2ex, topsep=0.5\topsep}

\usepackage{setspace}
\usepackage{ellipsis}
\usepackage{xspace}
\usepackage{hfoldsty}

\usepackage{tcolorbox}
\usepackage{lmodern}
\usepackage[tt=false]{libertine}
\usepackage[varqu]{zi4}
\usepackage[libertine]{newtxmath}

\usepackage{graphicx}
\usepackage{bbm}
\usepackage{subcaption}
\usepackage{amssymb}
\usepackage{algorithm}
\usepackage{amsthm}
\usepackage{algpseudocode}
\usepackage{diagbox}
\usepackage{multirow}
\usepackage{algorithm}
\usepackage{algorithmicx}
\usepackage{adjustbox}
\usepackage{tikz}
\usepackage{colortbl}
\usepackage{nicefrac}

\algnewcommand\algorithmicforeach{\textbf{for each}}
\algdef{S}[FOR]{ForEach}[1]{\algorithmicforeach\ #1\ \algorithmicdo}
\algnewcommand\algorithmicparallelforeach{\textbf{in parallel for each}}
\algdef{S}[FOR]{ParallelForEach}[1]{\algorithmicparallelforeach\ #1\ \algorithmicdo}
\algnewcommand{\algorithmicand}{\textbf{ and }}
\algnewcommand{\algorithmicor}{\textbf{ or }}
\algnewcommand{\OR}{\algorithmicor}
\algnewcommand{\AND}{\algorithmicand}
\algnewcommand\algorithmicnot{\textbf{not}}
\algdef{SE}[IF]{IfNot}{EndIf}[1]{\algorithmicif\ \algorithmicnot\ #1\ \algorithmicthen}{\algorithmicend\ \algorithmicif}%
\let\oldReturn\Return
\renewcommand{\Return}{\State\oldReturn}
\algtext*{EndWhile}%
\algtext*{EndIf}%
\algtext*{EndFor}%
\algtext*{EndFunction}%

\usepackage[
pdfa,
hidelinks,
pdftex, 
pdfdisplaydoctitle,
pdfpagelabels,
pdfauthor={},
pdftitle={},
pdfsubject={},
pdfkeywords={},
pdfproducer={Latex with the hyperref package},
pdfcreator={pdflatex}
]{hyperref}

\makeatletter
\def\thmt@refnamewithcomma #1#2#3,#4,#5\@nil{%
	\@xa\def\csname\thmt@envname #1utorefname\endcsname{#3}%
	\ifcsname #2refname\endcsname
	\csname #2refname\expandafter\endcsname\expandafter{\thmt@envname}{#3}{#4}%
	\fi
}
\makeatother
\usepackage[capitalise,noabbrev]{cleveref}

\usepackage{mathtools}
\DeclareMathOperator{\depth}{depth}
\DeclareMathOperator{\children}{chi}

\DeclareMathOperator{\height}{hgt}
\DeclareMathOperator{\ahu}{c_\mathsf{ahu}}
\DeclareMathOperator{\ahut}{\hat{c}_\mathsf{ahu}}
\newcommand{\wl}{c_{\mathsf{wl}}}
\DeclarePairedDelimiter\multiset{\lbrace\!\!\lbrace}{\rbrace\!\!\rbrace}%

\usepackage{lipsum}

\usepackage[auth-lg]{authblk}

\recalctypearea

\begin{document}
\title{On the Two Sides of Redundancy \\ in Graph Neural Networks}
	
	\author[1,2]{Franka Bause}
	\author[1,2]{Samir Moustafa}
	\author[3]{Johannes Langguth}
	\author[1]{Wilfried~N.~Gansterer}
	\author[1,4]{Nils M.~Kriege}
	
	\affil[1]{Faculty of Computer Science, University of Vienna, Vienna, Austria}
	\affil[2]{UniVie Doctoral School Computer Science, University of Vienna, Vienna, Austria}
	\affil[3]{Simula Research Laboratory, Oslo, Norway}
	\affil[4]{Research Network Data Science, University of Vienna, Vienna, Austria}
	\affil[ ]{\ttfamily\{firstname.lastname\}@univie.ac.at}

\date{\vspace{-30pt}}

\maketitle

\begin{abstract}
Message passing neural networks iteratively generate node embeddings by aggregating information from neighboring nodes. With increasing depth, information from more distant nodes is included. However, node embeddings may be unable to represent the growing node neighborhoods accurately and the influence of distant nodes may vanish, a problem referred to as oversquashing.
Information redundancy in message passing, i.e., the repetitive exchange and encoding of identical information amplifies oversquashing. We develop a novel aggregation scheme based on neighborhood trees, which allows for controlling redundancy by pruning redundant branches of unfolding trees underlying standard message passing. 
While the regular structure of unfolding trees allows the reuse of intermediate results in a straightforward way, the use of neighborhood trees poses computational challenges. We propose compact representations of neighborhood trees and merge them, exploiting computational redundancy by identifying isomorphic subtrees.
From this, node and graph embeddings are computed via a neural architecture inspired by tree canonization techniques.
Our method is less susceptible to oversquashing than traditional message passing neural networks and can improve the accuracy on widely used benchmark datasets.

\end{abstract}

\section{Introduction}
Graph neural networks (GNNs) have emerged as the dominant approach for machine learning on graph data, with the class of message passing neural networks (MPNNs)~\citep{Gil+2017} being widely-used. These networks update node embeddings layer wise by combining the current embedding of a node with those of its neighbors, involving learnable parameters. Suitable neural architectures, which admit a parametrization such that each layer represents an injective function uniquely encoding the input, have the same expressive power as the Weisfeiler-Leman algorithm~\citep{DBLP:journals/corr/abs-1810-00826}.
The Weisfeiler-Leman algorithm distinguishes two nodes if and only if the unfolding trees representing their neighborhoods are non-isomorphic. These unfolding trees correspond to the computational trees of MPNNs~\citep{extraforrev,Jegelka2022GNNtheory}.
Hence, nodes with isomorphic unfolding trees will obtain the same embedding, while for nodes with non-isomorphic unfolding trees, there exist parameters such that their embeddings differ. This implies that deeper unfolding trees lead to more expressive methods.
Despite this theoretical connection, shallow MPNNs are often favored in practice. Challenges arise from the observed convergence of node embeddings for deep architectures, referred to as \emph{oversmoothing}~\citep{Li2018,Liu2020}, and the issue of \emph{oversquashing}~\citep{DBLP:conf/iclr/0002Y21}, where the neighborhood of a node grows exponentially with the number of layers, and therefore, cannot be supposed to be accurately represented by a fixed-sized embedding. Recently, oversquashing has been investigated by analyzing the sensitivity of node embeddings to the initial features of distant nodes, relating the phenomenon to the \emph{graph curvature}~\cite{ToppingGC0B22}, the \emph{effective resistance}~\cite{Black23} and the \emph{commute time}~\cite{DBLP:conf/icml/GiovanniGBLLB23}. On this basis several graph rewiring strategies have been proposed to mitigate oversquashing~\cite{ToppingGC0B22,Black23}.

We address the problem of oversquashing by modifying the message passing scheme for eliminating the encoding of repeated information. For example, in an undirected graph, %
when a node sends information to its neighbour, future messages sent back via the same edge will contain the exact information previously sent, leading to redundancy.
In the context of walk-based graph learning this problem is well-known and referred to as \emph{tottering}~\citep{Mah'e2004}. 
Recent work by \citet{NEURIPS2022_1bd6f176} established a first result formalizing the relation between redundancy and oversquashing by sensitivity analysis.
Several recent GNNs replace the walk-based aggregation by mechanisms based on simple or shortest paths reporting promising results~\citep{AbboudDC22,michel_expressive_2023,DBLP:conf/kdd/JiaLYYLA20}. PathNNs~\citep{michel_expressive_2023} and RFGNN~\citep{DBLP:conf/kdd/JiaLYYLA20} are closely related approaches, defining path-based trees for nodes and employing custom aggregation schemes. However, these methods suffer from high computational costs compared to standard MPNNs and often have an exponential time complexity. The crucial advantage of MPNNs is the regular structure of aggregations applied through all layers, while reducing information redundancy leads to a less regular structure, rendering it challenging to exploit computational redundancy.

\paragraph{Our contribution.}
We systematically explore the issue of information redundancy within MPNNs and introduce principled techniques to eliminate superfluous messages. Our investigation is based on the implicit tree representation used by both MPNNs and the Weisfeiler-Leman algorithm. We first develop a neural tree canonization approach that systematically processes trees in a bottom-up fashion and extend it to directed acyclic graphs (DAGs). To exploit computational redundancy, we merge multiple trees representing node neighborhoods into a single DAG identifying isomorphic subtrees.
Our approach, termed DAG-MLP, recovers the computational graph of MPNNs for unfolding trees, while avoiding redundant computations in the presence of symmetries.
We employ the canonization technique on \emph{neighborhood trees}, which are derived from unfolding trees by eliminating nodes that appear multiple times. We show that neighborhood trees allow distinguishing nodes and graphs that are indistinguishable by the Weisfeiler-Leman algorithm. 
The DAGs derived from neighborhood trees have size at most $O(nm)$ for input graphs with $n$ nodes and $m$ edges making the approach computational feasible. 
We formally show by sensitivity analysis that our approach reduces oversquashing.
Our approach achieves high accuracy across various node and graph classification tasks.

\section{Related Work}

\begin{table}[tb]
    \centering
     \caption{Time complexity of preprocessing, size of computation graph and expressivity of our method compared to related work. $n$: number of nodes,
$m$: number of edges,
$b$: maximum node degree,
$K$: path length,
$h$: tree height,
$L$: number of layers, and
$m_2 = 0.5 \sum_{v \in V} |N_2(v)|$.}
    \label{tab:complexity}
    \begin{tabular}{cccc}
    \toprule
    \textbf{Method} & \textbf{Preprocessing} & \textbf{Size Comp. Graph/Runtime} & \textbf{Expressivity} \\
    \midrule
    PathNet & $O(mb)$ &$O(2^L(m+m_2))$&n/a\\
    PathNN-SP & $O(nb^K)$& $O(nbK)$ & incomparable\\
    PathNN-SP+ & $O(nb^K)$& $O(nbK)$ &$>1$-WL\\
    RFGNN & $O(\nicefrac{n!}{(n-h-1)!})$&$O(\nicefrac{n!}{(n-h-1)!})$&incomparable\\
    \midrule
    DAG-MLP ($0$-/$1$-NTs) & $O(nm)$  & $O(nm)$&incomparable\\
    \bottomrule
    \end{tabular}
\end{table}

The graph isomorphism network (GIN)~\citep{DBLP:journals/corr/abs-1810-00826} is an MPNN that generalizes the Weisfeiler-Leman algorithm, achieving its expressive power.
The limited expressivity of simple MPNNs has led to an increased interest in researching more powerful architectures, such as encoding graph structure as additional features or modifying the message passing procedure. 
Shortest Path Networks~\citep{AbboudDC22} use multiple aggregation functions for different shortest path lengths, %
allowing direct communication with distant nodes. While this might help mitigate oversquashing, information about the structure of the neighborhood can still not be represented adequately and the gain in expressivity is limited.
Distance Encoding GNNs~\citep{LiWWL20} encode the distances of nodes to a set of target nodes. While being provably more expressive than the standard WL algorithm, the approach is limited to solving node-level tasks, as the encoding depends on a fixed set of target nodes and has not been employed for graph-level tasks.
MixHop~\citep{Abu-El-HaijaPKA19} concatenates results from activation functions for each neighborhood, but in contrast to Shortest Path Networks~\citep{AbboudDC22}, the aggregation is based on normalized powers of the adjacency matrix, not shortest paths, which fails to solve the problem of redundant messages.
SPAGAN~\citep{YangWSYT19} proposes a path-based attention mechanism, sampling shortest paths and using them as features. However, a theoretical investigation is lacking and the approach utilizes only one layer.
Short-rooted random walks in~\citep{sun2022beyond} capture long-range dependencies, but have notable limitations due to sampling paths. The evaluation is restricted to node classification datasets and an extensive study of their expressive power is lacking.
IDGNN~\citep{YouGYL21} tracks the identity of the root node in unfolding trees, achieving higher expressivity than $1$-WL, %
but failing to reduce redundant information aggregation.
PathNNs~\citep{michel_expressive_2023} define path-based trees and a custom aggregation scheme, but overlook exploiting computational redundancy.
RFGNNs~\citep{NEURIPS2022_1bd6f176} aim to reduce redundancy by altering the message flow to only include each node (except for the root node) at most once in each path of the computational tree. While this reduces redundancy to some extent, nodes and even the same subpaths may repeatedly occur in the computational trees.
The redundancy in computation is not addressed resulting in a highly inefficient preprocessing and computation, limiting the method to a maximum of $3$ layers in the experiments. We discuss further differences between our approach and RFGNN in Appendix~\ref{sec:rfgnn}.

These architectures lack thorough investigation of their expressivity and connections to other approaches. Importantly, they do not explicitly investigate both types of redundancy in MPNNs -- redundancy in the information flow and in computation. We compare the time complexity, as well as the expressivity of our method DAG-MLP and other relevant methods in Table~\ref{tab:complexity} and further discuss it in Section~\ref{sec:complexity}.

\section{Preliminaries}
In this section, we provide an overview of essential definitions and the notation used throughout this article, accompanied by the introduction of fundamental techniques.

\begin{figure}[tb]
	\centering
	\begin{subfigure}{0.19\linewidth}
		\includegraphics[height=0.06\textheight]{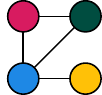}
	\end{subfigure}
	\begin{subfigure}{0.19\linewidth}
		\includegraphics[height=0.06\textheight]{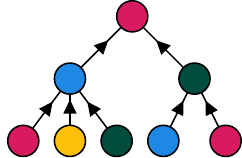}
	\end{subfigure}
	\begin{subfigure}{0.19\linewidth}
		\includegraphics[height=0.06\textheight]{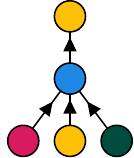}
	\end{subfigure}
	\begin{subfigure}{0.19\linewidth}
		\includegraphics[height=0.06\textheight]{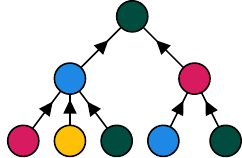}
	\end{subfigure}
	\begin{subfigure}{0.19\linewidth}
		\includegraphics[height=0.06\textheight]{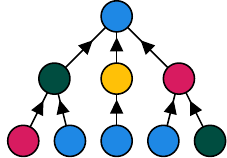}
	\end{subfigure}
	\caption{Graph $G$ and its unfolding trees $F_2^v$ for all $v\in V(G)$.}
	\label{fig:unfoldingtree}
\end{figure}

\paragraph{Graph theory.}
A \emph{graph} $G = (V,E,\mu,\nu)$ consists of a set of vertices $V$, a set of edges $E \subseteq V\times V$ between them, and functions $\mu \colon V \to X$ and $\nu\colon E \to X$ assigning arbitrary attributes to the vertices and edges, respectively.\footnote{Edge attributes are not considered in the following for clarity of presentation, though the proposed methods can be extended to incorporate them.}
An edge from $u$ to $v$ is denoted by $uv$, and in undirected graphs $uv=vu$.
The vertices and edges of a graph $G$ are denoted by $V(G)$ and $E(G)$, respectively. The \emph{neighbors} (or in-neighbors) of a vertex $u \in V$ are denoted by $N(u) = \{v\mid  vu\in E \}$, and the \emph{out-neighbors} of a vertex $u \in V$ are denoted by $N_o(u) = \{v\mid  uv\in E\}$.
A \emph{multigraph} is a graph, where $E$ is a multiset, allowing multiple edges between a pair of vertices.
Two graphs $G$ and $H$ are isomorphic, denoted by $G\simeq H$, if there exists a bijection $\phi\colon V(G) \to V(H)$, such that $\forall u,v \in V(G)\colon \mu(v)= \mu(\phi(v)) \land uv \in E(G) \Leftrightarrow \phi(u)\phi(v)\in E(H) \land \forall uv \in E(G)\colon \nu(uv)=\nu(\phi(u)\phi(v))$. We refer to $\phi$ as an \emph{isomorphism} between $G$ and $H$.

An \emph{in-tree} $T$ is a connected, directed, acyclic graph with a distinct vertex $r \in V(T)$ with no outgoing edges, referred to as \emph{root} ($r(T)$), in which $\forall v \in V(T)\backslash r(T): |N_o(v)|=1$. For $v \in V(T)\backslash r(T)$ the \emph{parent} $p(v)$ is the unique vertex $u \in N_o(v)$, and $\forall v \in V(T)$ the \emph{children} are defined as $\children(v)= N(v)$. 
We refer to vertices without incoming edges as \emph{leaves}, denoted by $l(T)=\{v \in V(T) \mid \children(v) = \emptyset\}$. Conceptually an in-tree is a directed tree, in which there is a unique directed path from each vertex to the root~\citep{intree}. In our paper, we only consider in-trees and will therefore refer to them simply as \emph{trees}. 
In-trees are generalized by directed, acyclic graphs (DAGs).
The \emph{leaves} of a DAG $D$ and the \emph{children} of a vertex are defined as in trees. However, there can be multiple roots, and a vertex may have more than one parent. We refer to all vertices in $D$ without outgoing edges as \emph{roots}, denoted by $r(D)=\{v \in V(D) \mid N_o(v) = \emptyset\}$, and define the \emph{parents} $p(v)$ of a vertex $v$ as $p(v)=N_o(v)$.
The height $\height$ of a node $v$ is the length of the longest path from any leaf to $v$: $\height(v) = 0\text{, if } v \in l(D) \text{ and }
\height(v) = \max_{c \in \children(v)} \height(c) + 1\text{, otherwise.} $
The height of a DAG $D$ is defined as $\height(D)=\max_{v \in V(D)} \height(v)$.
For clarity we refer to the vertices of a DAG as nodes to distinguish them from the graphs that are the input of a graph neural network.

\paragraph{Weisfeiler-Leman and Message Passing Neural Networks.}\label{sec:wl_unfolding}
The 1-dimensional Weisfeiler-Leman (WL) algorithm, also known as \emph{color refinement}, starts with vertices having a color corresponding to their label (or a uniform coloring for unlabeled vertices). 
In each iteration the vertex color is updated based on the multiset of colors of its neighbors according to
\begin{equation*}
\wl^{(i+1)}(v) = h\left(\wl^{(i)}(v),\multiset{\wl^{(i)}(u)\mid u \in N(v)}\right) \quad\forall v \in V(G),
\end{equation*}
where $h$ is an injective function, typically using integers to represent colors.

The color of a vertex encodes its neighborhood through a tree $T$, which may contain multiple representatives of each vertex. Let $\phi\colon V(T)\to V(G)$ be a mapping such that $\phi(n)=v$ if the node $n$ in $V(T)$ represents the vertex $v$ in $V(G)$.
The \emph{unfolding tree} $F_i^v$ with height $i$ of the vertex $v \in V(G)$ consists of a root $n_v$ with $\phi(n_v)=v$ and child subtrees $F^u_{i-1}$ for all $u \in N(v)$, where $F_0^v = (\{n_v\}, \emptyset)$.
The attributes of the original graph are preserved, as illustrated in Figure~\ref{fig:unfoldingtree}.
The unfolding trees $F_i^v$ and $F_i^w$ of two vertices $v$ and $w$ are isomorphic if and only if $\wl^{(i)}(v)=\wl^{(i)}(w)$.

\begin{figure}[tb]
\centering
\begin{subfigure}{0.19\linewidth}\centering
	\includegraphics[height=0.06\textheight]{examplegraph}
	\subcaption{Graph $G$}
\end{subfigure}
\begin{subfigure}{0.19\linewidth}\centering
	\includegraphics[height=0.06\textheight]{utree1}
	\subcaption{$F_2^v$}
\end{subfigure}
\begin{subfigure}{0.19\linewidth}\centering
	\includegraphics[height=0.06\textheight]{ntree0_1}
	\subcaption{$T_{2,0}^v$}
\end{subfigure}
\begin{subfigure}{0.19\linewidth}\centering
	\includegraphics[height=0.06\textheight]{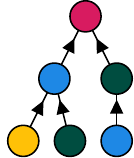}
	\subcaption{$T_{2,1}^v$}
\end{subfigure}
\caption{Graph $G$ and the unfolding, $0$- and $1$-redundant neighborhood trees of height $2$ of vertex $v$ (vertex in the upper left of $G$).}
\label{fig:exampletrees}
\end{figure}

\emph{Message passing neural networks} such as GIN~\citep{DBLP:journals/corr/abs-1810-00826} can be seen as a neural variant of the Weisfeiler-Leman algorithm. The embedding of a vertex $v$ in layer $i$ of GIN is defined as
\begin{equation}\label{eq:gin}
x_i(v) = \text{MLP}_i \left(\left( 1 + \epsilon_i \right) \cdot x_{i-1}(v) + \sum_{u\in  N(v)} x_{i-1}(u) \right),
\end{equation}
where the initial features $x_0(v)$ are usually acquired by applying a multi-layer perceptron (MLP) to the vertex features.

\section{Non-Redundant Graph Neural Networks}
We propose to restrict the information flow in message passing to regulate redundancy through the use of $k$-redundant neighborhood trees. We first develop a neural tree canonization technique, and obtain an MPNN via its application to unfolding trees.
Subsequently, we explore computational methods on graph level, reusing information computed for subtrees and derive a customized GNN architecture.
Finally, we prove that $k$-redundant neighborhood trees and unfolding trees are incomparable regarding their expressivity on node-level.

\subsection{Removing Information Redundancy}\label{sec:canon}
As previously discussed, two vertices obtain the same WL color if and only if their unfolding trees are isomorphic. This concept directly carries over to message passing neural networks and their computational tree~\citep{extraforrev,Jegelka2022GNNtheory}.
However, unfolding trees were mainly used as tools in expressivity analysis and as a conceptual framework for explaining mathematical properties in graph learning~\citep{Kriege2016b,NikolentzosCV23}.
We propose a novel perspective on MPNNs through tree canonization. From this perspective, we derive a non-redundant GNN architecture based on neighborhood trees.

In their classical textbook, \citet*[Section 3.2]{AhoHU74} describe a linear time isomorphism test for rooted unordered trees, detailed in Appendix~\ref{sec:ahu}.
We give a high-level description to establish the foundation for our neural variant without focusing on the running time.
The algorithm proceeds in a bottom-up manner, assigning integers $\ahu(v)$ to each node $v$ in the tree. 
The function $f$ injectively maps a pair consisting of an integer and a multiset of integers to a new, unused integer. 
Initially, all leaves $v$ are assigned integers $\ahu(v)=f(\mu(v),\emptyset)$ based on their label $\mu(v)$. Then internal nodes are processed level-wise in a bottom-up manner, ensuring that whenever a node is processed, all its children have been considered. 
Hence, the algorithm computes for all nodes $v$ of the tree
\begin{equation}\label{eq:canon}
\ahu(v)= f(\mu(v), \multiset{ \ahu(u) \mid u \in \children(v) }).
\end{equation}
This process ensures the unique representation of non-isomorphic trees, serving as the foundation of our neural tree canonization technique.

\paragraph{GNNs via unfolding tree canonization.}
We combine Eq.~\eqref{eq:canon} and the definition of unfolding trees, denoting the root of an unfolding tree of height $i$ of a vertex $v$ by $n^i_v$. This yields
\begin{equation}\label{eq:wl_gnn}
\ahu(n^i_v)
=f(\mu(n^i_v), \multiset{ \ahu(n^{i-1}_u) \mid n^{i-1}_u \in \children(n^i_v) })
=f(\mu(v), \multiset{ \ahu(n^{i-1}_u) \mid u \in N(v) }).
\end{equation}
By implementing the function $f$ using a suitable neural architecture and replacing its codomain with embeddings in $\mathbb{R}^d$, we readily obtain a GNN based on our canonization approach.
The key difference to standard GNNs is that the first component of the pair in Eq.~\eqref{eq:wl_gnn} is the initial vertex feature instead of the embedding from the previous iteration.
Utilizing the technique proposed by~\citet{DBLP:journals/corr/abs-1810-00826} and replacing the first addend in Eq.~\eqref{eq:gin} with the initial embedding, we formulate the \emph{unfolding tree canonization GNN}
\begin{equation}\label{eq:canongin}
x_i(v) = \text{MLP}_i \left(\left(1 + \epsilon_i \right) \cdot x_{0}(v) + \sum_{u\in  N(v)} x_{i-1}(u) \right).
\end{equation}
It is established that MPNNs cannot distinguish two vertices with the same WL color or unfolding tree. Given that the function $\ahu(n^i_v)$ uniquely represents the unfolding tree for an injective function $f$, realizable by Eq.~\eqref{eq:canongin}~\citep{DBLP:journals/corr/abs-1810-00826}, we infer the following proposition.
\begin{proposition}
Unfolding tree canonization GNNs, as defined in Eq.~\eqref{eq:canongin}, are as expressive as GIN, as defined in Eq.~\eqref{eq:gin}.
\end{proposition}
Despite the equivalence in expressivity, the canonization-based approach avoids redundancy since $x_{i-1}(v)$ represents the entire unfolding tree rooted at $v$ of height $i-1$, while using the initial vertex features $x_{0}(v)$ is sufficient. We proceed by investigating redundancy within unfolding trees themselves.

\paragraph{GNNs via neighborhood tree canonization.}
We leverage the concept of neighborhood trees to manage redundancy in unfolding trees.\footnote{In a parallel work, neighborhood trees were investigated for approximating the graph edit distance~\cite{NTreeGED}. %
}
A $k$-redundant neighborhood tree ($k$-NT) $T_{i,k}^v$ is derived from the unfolding tree $F_{i}^v$ by removing all subtrees with roots that occurred more than $k$ levels before (seen from root to leaves). Here, $\depth(v)$ denotes the length of the path from $v$ to the root, and $\phi(v)$ denotes the original vertex in $V(G)$ represented by $v$ in the unfolding or neighborhood tree. 

\begin{definition}[$k$-redundant Neighborhood Tree]\label{def:nt}
For $k\geq 0$, the \emph{$k$-redundant neighborhood tree} ($k$-NT) of a vertex $v \in V(G)$ with height $i$, denoted by $T_{i,k}^v$, is defined as the subtree of the unfolding tree $F_{i}^v$ induced by the nodes $u \in V(F_{i}^v)$ satisfying
\begin{equation*}
\forall w \in V(F_{i}^v)\colon\phi(u)=\phi(w) \Rightarrow  \depth(u)\leq \depth(w)+k.
\end{equation*}
\end{definition}
Figures~\ref{fig:exampletrees} and~\ref{fig:ntree} provide examples of unfolding and neighborhood trees.
It is worth noting that for $k\geq i$ the $k$-redundant neighborhood tree is equivalent to the WL unfolding tree.
\begin{figure}[tb]
\centering
\begin{subfigure}{0.19\linewidth}
	\includegraphics[height=0.05\textheight]{examplegraph}
\end{subfigure}
\begin{subfigure}{0.19\linewidth}
	\includegraphics[height=0.05\textheight]{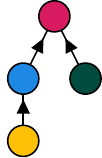}
\end{subfigure}
\begin{subfigure}{0.19\linewidth}
	\includegraphics[height=0.05\textheight]{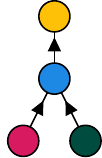}
\end{subfigure}
\begin{subfigure}{0.19\linewidth}
	\includegraphics[height=0.05\textheight]{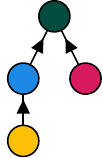}
\end{subfigure}
\begin{subfigure}{0.19\linewidth}
	\includegraphics[height=0.05\textheight]{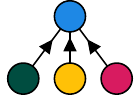}
\end{subfigure}
\caption{Graph $G$ and its $0$-NTs $T_{2,0}^v$ for all $v\in V(G)$.}
\label{fig:ntree}
\end{figure}

We can directly apply the neural tree canonization technique to neighborhood trees. However, a simplifying expression based on the neighbors in the input graph, as given by Eq.~\eqref{eq:wl_gnn} for unfolding trees, is not possible for neighborhood trees. Therefore, we explore techniques to systematically exploit computational redundancy.

\subsection{Removing Computational Redundancy} \label{sec:comp_redundancy}
The computation DAG of an MPNN involves the embedding of a set of trees representing the vertex neighborhoods of a single or multiple graphs. Results computed for one tree can be reused for others by identifying isomorphic substructures, thereby minimizing computational redundancy. We first describe how to merge trees in a general context and then discuss its application to unfolding and neighborhood trees.

\begin{figure}[tb]
	\centering
	\begin{subfigure}{0.275\linewidth}\centering
		\includegraphics[width=0.65\textwidth]{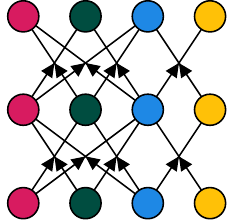}
		\subcaption{Merged unfolding trees}
		\label{fig:unfoldingDAG}
	\end{subfigure}
	\begin{subfigure}{0.225\linewidth}\centering
		\includegraphics[width=0.8\textwidth]{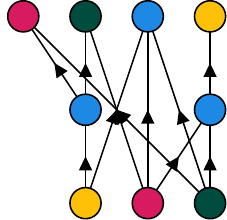}
		\subcaption{Merged $0$-NTs}
		\label{fig:ntreeDAG}
	\end{subfigure}
	\begin{subfigure}{0.225\linewidth}\centering
		\includegraphics[width=0.8\textwidth]{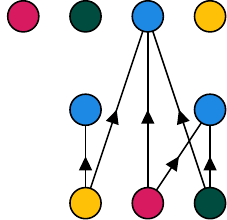}
		\subcaption{$\mathcal{E}_1$}\label{fig:e1}
	\end{subfigure}
	\begin{subfigure}{0.225\linewidth}\centering
		\includegraphics[width=0.8\textwidth]{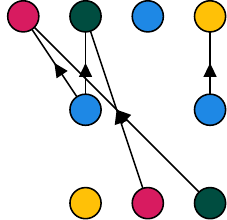}
		\subcaption{$\mathcal{E}_2$}\label{fig:e2}
	\end{subfigure}
	\caption{Computation DAGs for unfolding (\subref{fig:unfoldingDAG}) and $0$-NTs (\subref{fig:ntreeDAG}) of height $2$ of graph $G$. And edges in the different layers of the merge DAG of $0$-NTs (\subref{fig:e1}), (\subref{fig:e2}).}
	\label{fig:compgraph}
\end{figure}

\paragraph{Merging trees into a DAG.}
\label{subsec:mergetrees}
The neural tree canonization approach developed in the last section can be directly applied to DAGs. Given a DAG $D$, it computes an embedding for each node $n$ in $D$ that represents the tree $F_{n}$ obtained by recursively following its children, similar as in unfolding trees, cf.~Section~\ref{sec:wl_unfolding}. 
Since $D$ is acyclic, the height of $F_{n}$ is bounded.
A detailed description of a neural architecture is postponed to Section~\ref{sec:dagmlp}.

Given a set of trees $\mathcal{T}=\{T_1, \dots, T_n\}$, a \emph{merge DAG} of $\mathcal{T}$ is a pair $(D,\xi)$, where $D$ is a DAG, $\xi\colon\{1,\dots n\} \to V(D)$ is a mapping, and for all $i \in \{1,\dots,n\}$ we have $T_i \simeq F_{\xi(i)}$.
The definition guarantees that the neural tree canonization approach applied to the merge DAG produces the same result for the nodes in the DAG as for the nodes in the original trees.
A trivial merge DAG is the disjoint union of the trees with $\xi(i)=r(T_i)$. However, depending on the structure of the given trees, we can identify the subtrees they have in common and represent them only once, such that two nodes of different trees share the same child, resulting in a DAG instead of a forest.

We propose an algorithm that builds a merge DAG by successively adding trees to an initially empty DAG, creating new nodes only when necessary. Our approach maintains a canonical labeling for each node of the DAG and computes a canonical labeling for each node of the tree to be added using the AHU algorithm (cf.~Appendix~\ref{sec:ahu}). Then, the tree is processed starting at the root. If the canonical labeling of the root is present in the DAG, then the algorithm terminates. Otherwise, the subtrees rooted at its children are inserted into the DAG by recursive calls. Finally, the root is created and connected to the representatives of its children in the DAG.
We introduce a node labeling $L\colon V_T \to \mathcal{O}$ used for tree canonization, where $V_T=\bigcup_{i=1}^n V(T_i)$ and $\mathcal{O}$ an arbitrary set of labels, refining the original node attributes, i.e., $L(u)=L(v)\Rightarrow \mu(u)=\mu(v)$ for all $u,v$ in $V_T$. When $\mathcal{O}$ consists of integers from the range $1$ to $|V_T|$, the algorithm runs in $O(|V_T|)$ time (see Appendix~\ref{sec:mergetrees_appendix} for details).
When two siblings that are the roots of isomorphic subtrees are merged, this leads to parallel edges in the DAG. Parallel edges can be avoided by using a labeling satisfying $L(u)=L(v) \Rightarrow \mu(u)=\mu(v) \wedge p(u) \neq p(v)$ for all $u,v$ in $V_T$. 

Unfolding trees and $k$-NTs can grow exponentially in size with increasing height. However, this is not case for merge DAGs obtained by the algorithm described above, as we will show below. Moreover, we can directly generate DAGs of size $O(m \cdot (k + 1))$ representing individual $k$-NTs with unbounded height in a graph with $m$ edges (see Appendix~\ref{sec:compacttrees} for details).

\paragraph{Merging unfolding trees.}
Merging the unfolding trees of a graph with the labeling $L=\phi$ leads to the computation DAG of GNNs. Figure~\ref{fig:unfoldingDAG} shows the computation DAG for the graph from Figure~\ref{fig:unfoldingtree}. The roots in this DAG correspond to the representation of the vertices after aggregating information from the lower layers.
Each vertex occurs once at every layer of the DAG, and the links between any two consecutive layers are given by the adjacency matrix of the original graph. While this allows computation based on the adjacency matrix widely-used for MPNNs, it involves the encoding of redundant information.
Our method has the potential to compress the computational DAG further by using the less restrictive labeling $L=\mu$, leading to a DAG where at layer $i$ all vertices $u,v$ with $\wl^{(i)}(u)=\wl^{(i)}(v)$ are represented by the same node. This compression appears particularly promising for graphs with symmetries.

\paragraph{Merging neighborhood trees.}
When merging $k$-redundant neighborhood trees in the same way using the labeling $L=\mu$ (or $L=\phi$ to avoid parallel edges), it results in a computation DAG with a more irregular structure, as illustrated in Figure~\ref{fig:ntreeDAG}. Firstly, there might be multiple nodes at the same level representing the same original vertex. Secondly, the adjacency matrix of the original graph cannot be used to propagate the information. A straightforward upper bound on the size of the merge DAG for a graph with $n$ nodes and $m$ edges is $O(nmk + nm)$.

We apply the neural tree canonization approach to the merge DAG in a bottom-up fashion, starting from the leaves and progressing to the roots. Each edge is used exactly once in this computation. 
Let $D=(\mathcal{V},\mathcal{E})$ be a merge DAG.
The nodes can be partitioned based on their height, leading to
$\mathcal{L}_i = \{v \in \mathcal{V} \mid \height(v) = i\}$.
This induces an edge partition $\mathcal{E}_i = \{uv \in \mathcal{E} \mid v \in \mathcal{L}_{i}\}$, where all edges with the same end node $v$ are in the same layer, and all incoming edges of children of $v$ belong to a previous layer. Note that, since $\mathcal{L}_0$ contains all leaves of the DAG, there is no $\mathcal{E}_{0}$.
Figures~\ref{fig:e1} and~\ref{fig:e2} depict the edge sets $\mathcal{E}_1$ and $\mathcal{E}_2$ for the example merge DAG illustrated in Figure~\ref{fig:ntreeDAG}.

\subsection{Non-Redundant Neural Architecture (DAG-MLP)}\label{sec:dagmlp}
We propose a neural architecture to compute embeddings for nodes in a merge DAG, allowing to retrieve embeddings of the contained trees from their roots. The process involves a preprocessing step that transforms the node labels, using $\text{MLP}_{0}$ to map them to an embedding space of fixed dimensions. Subsequently, an $\text{MLP}_{i}$ is used to process nodes at each layer $\mathcal{L}_i$.
\begin{align*}
\mu^\prime(v) &= \text{MLP}_0\left(\mu(v)\right), & \forall v \in \mathcal{V}  \\
x(v) &= \mu^\prime(v), & \forall v \in \mathcal{L}_0 \\
x(v) &= \text{MLP}_i \left(\left( 1 + \epsilon_i \right) \cdot \mu^\prime(v) + \sum_{\forall u\colon  uv \in \mathcal{E}_i} x(u) \right), & \forall v \in \mathcal{L}_i, i\in\{1,\dots,n\}   
\end{align*}
The DAG-MLP can be computed through iterated matrix-vector multiplication analogous to standard GNNs. Let $\mathbf{L}_i$ be a square matrix with ones on the diagonal at position $j$ if $v_j \in \mathcal{L}_i$, and zeros elsewhere. Let $\mathbf{E}_i$ represent the adjacency matrix of $(\mathcal{V}, \mathcal{E}_i)$, and let $\mathbf{F}$ denote the node features of $\mathcal{V}$, corresponding to the initial node labels. The transformed features $\mathbf{F^\prime}$ are obtained through $\text{MLP}_0$, and $\mathbf{X}^{[i]}$ represents the updated embeddings at layer $i$ of the DAG.
\begin{alignat*}{2}
\label{equ:dag_mlp_matrix_formulation}
\mathbf{F^\prime} &= \text{MLP}_0\left(\mathbf{F}\right), \quad
\mathbf{X}^{[0]}  \;= \mathbf{L}_0 \mathbf{F^\prime}, \\
\mathbf{X}^{[i]}  &= \text{MLP}_i \left(\left( 1 + \epsilon_i \right) \cdot \mathbf{L}_i \mathbf{F^\prime} + \mathbf{E}_i \mathbf{X}^{[i-1]}\right) + \mathbf{X}^{[i-1]}\,,   
\end{alignat*}
In the above equation, $\text{MLP}_i$ is applied to the rows associated with nodes in $\mathcal{L}_i$. The embeddings $\mathbf{X}^{[i]}$ are initialized to zero for inner nodes and computed level-wise. To preserve embeddings from all previous layers, we add $\mathbf{X}^{[i-1]}$ during the computation of $\mathbf{X}^{[i]}$.
Suppose the merge DAG $(D,\xi)$ contains the trees $\{T_1, \dots, T_n\}$. We obtain a node embedding $\mathbf{X}^{[n]}_{\xi(i)}$ for each tree $T_i$ with $i\in\{1,\dots, n\}$.
This approach allows for obtaining the final embedding for a vertex by using a single tree (Fixed Single-Height) or combining trees of different heights, for example all NTs of size up to a certain maximum (Combine Heights). %
Further details on the resulting architecture are described in Appendix~\ref{sec:dag_mlp_graph_class_architecture}.

\subsection{Expressivity of k-NTs}
\label{sec:1nt_expressivity}

Here, we investigate how expressive $k$-NTs are compared to unfolding trees.
While it is evident that $k$-NTs are a node invariant, providing the same result for nodes that can be mapped to each other by an isomophism or automorphism, they might also produce the same results for nodes that cannot. This means that, similar to unfolding trees, they are not a complete node invariant.

We show that there are vertices that $1$-WL cannot distinguish, but $k$-NTs can, and vice versa, proving that both methods are incomparable regarding their expressivity on node level. We further conjecture that, while $k$-NTs cannot distinguish certain vertices that $1$-WL can, they can still distinguish the graphs containing such vertices, making $k$-NTs more expressive on the graph level.

\begin{theorem}\label{thm:expr}
The expressivity of $k$-NT and unfolding trees is incomparable, i.e.,
\begin{align}
\text{1. }&\exists u,v\colon  F_\infty^u=  F_\infty^v \wedge T_{\infty,k}^u \neq T_{\infty,k}^v. \nonumber\\ 
\text{2. }&\exists u,v\colon  F_\infty^u \neq  F_\infty^v \wedge T_{\infty,k}^u = T_{\infty,k}^v \nonumber
\end{align}
\end{theorem}
\begin{proof}
We prove the statement by giving concrete examples. In Figure~\ref{ex:hex} an example for 1. is given: as commonly known, the Weisfeiler-Leman algorithm is unable to distinguish the two graphs, indicating identical unfolding trees of the vertices in the two graphs. However, for any $k$ the $k$-NT of the vertices will differ for $h \geq k+2$.
In Figure~\ref{ex:ce} an example for 2. (for $1$-NTs) is given: the unfolding trees of the two marked vertices differ, while the $1$-NTs are identical. This serves as an instance where $1$-WL  can distinguish vertices that $k$-NTs cannot.
\end{proof}

\begin{figure}[tb]
\centering
\begin{subfigure}[b]{0.2\linewidth}\centering
	\includegraphics[width=0.5\textwidth]{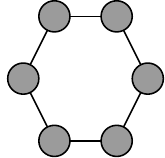}
	\subcaption{Hexagon}\label{fig:hexagon}
\end{subfigure}
\begin{subfigure}[b]{0.25\linewidth}\centering
	\includegraphics[width=0.5\textwidth]{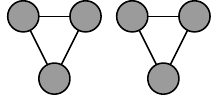}
	\subcaption{Two triangles}\label{fig:triangles}
\end{subfigure}
\begin{subfigure}[b]{0.3\linewidth}\centering
	\includegraphics[width=0.5\textwidth]{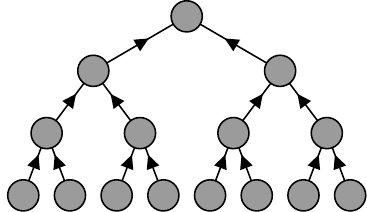}
	\subcaption{Unfolding trees}\label{fig:unfolding}
\end{subfigure}
\begin{subfigure}[b]{0.22\linewidth}\centering
	\includegraphics[width=0.5\textwidth]{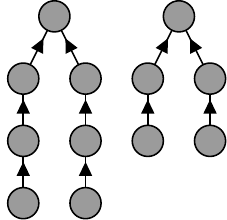}
	\subcaption{$1$-NTs}\label{fig:nts}
\end{subfigure}
\caption{Two graphs (\subref{fig:hexagon}), (\subref{fig:triangles}) that cannot be distinguished by unfolding trees, but by $k$-NTs. Figure~(\subref{fig:unfolding}) shows the unfolding tree $F_3$, which is the same for all vertices of both graphs, while (\subref{fig:nts}) shows the $1$-NTs of the vertices in the hexagon (left) and the triangle (right).}
\label{ex:hex}
\end{figure}

\begin{figure}[tb]
\centering
\begin{subfigure}[b]{0.2\linewidth}\centering
	\includegraphics[width=0.5\textwidth]{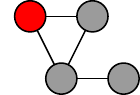}
	\subcaption{$G_1$}\label{fig:ce_g1}
\end{subfigure}
\begin{subfigure}[b]{0.25\linewidth}\centering
	\includegraphics[width=0.5\textwidth]{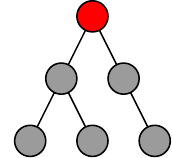}
	\subcaption{$G_2$}\label{fig:ce_g2}
\end{subfigure}
\begin{subfigure}[b]{0.24\linewidth}\centering
	\includegraphics[width=0.5\textwidth]{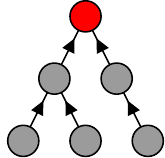}
	\subcaption{$1$-NTs of red vertices}\label{fig:ce_ntree}
\end{subfigure}
\caption{Two graphs (\subref{fig:ce_g1}), (\subref{fig:ce_g2}) in which the red vertices can be distinguished by unfolding trees, but not by $k$-NTs. Figure~(\subref{fig:ce_ntree}) shows the $1$-NTs of the red vertices, which are the same. However, $G_1$ and $G_2$ can be distinguished by their multisets of $1$-NTs.}
\label{ex:ce}
\end{figure}

We have shown that on node level, the expressivity of $k$-NTs and unfolding trees is incomparable. However, the examples, where $k$-NTs fail to distinguish nodes that $1$-WL can, can actually be distinguished on the graph level. This arises from the fact, that the graphs have a different number of vertices and the $k$-NTs of the other nodes differ.

\subsection{Computational Complexity, Expressivity and Oversquashing}
\label{sec:complexity}
The DAG representing the $k$-NT of a single vertex has a size in $O(mk+m)$, where $m$ is the number of edges in the graph. 
The lexicographic encoding and merging of $k$-NTs to generate the DAG can be done in time linear in its size.
A trivial upper bound on the size of the merge DAG of a graph with $n$ nodes and $m$ edges is $O(nmk + nm)$. Overall, this means that preprocessing can be done in $O(nmk)$ time, where $k$ can be considered constant. For $0$- and $1$-NTs, we obtain time $O(nm)$. Table~\ref{tab:complexity} compares the complexity and expressivity of our method to related work. 

Understanding and relating the expressivity of different approaches is non-trivial. In the case of PathNet, its expressivity concerning the WL-hierarchy remains unexplored. PathNN-SP+ has been shown to be more expressive than $1$-WL. While it is claimed that RFGNNs are maximally expressive, the proof claiming higher expressivity on node level as presented in~\citet[Lemma~7]{NEURIPS2022_1bd6f176} is not correct (rather it is incomparable on node level, as detailed in Appendix~\ref{sec:rfgnn}). Consequently, it remains uncertain, whether RFGNN is strictly more expressive on graph level. PathNN-SP~\cite{michel_expressive_2023} states that it can only disambiguate graphs at least as well as $1$-WL and is not strictly more powerful. This is because, due to sampling, isomorphic graphs could be mapped to different representations, indicating that it is not a graph invariant.

RFGNN and PathNN-SP+ involve enumerating all possible paths and all shortest paths, respectively. A straightforward example of a graph consisting of a chain of joined 4-cycles, shows that there is an exponential number of such paths, and that these paths still contain redundancy. In our approach, we avoid this redundancy by not explicitly building NTs, but instead generate DAGs, resulting in a much more compact representation (refer to Appendix~\ref{sec:compacttrees}).

In Appendix~\ref{app:oversquash}, we investigate the theoretical connection of our method and RFGNN in terms of relative influence, and show that $0$- and $1$-NTs can address the oversquashing problem more effectively.

To summarize, our method is the first to address both types of redundancy in GNNs, informational and computational redundancy, with polynomial running time, and can mitigate oversquashing better than comparable approaches.
Therefore, our method addresses a gap in GNN research that has not been previously considered.

\section{Experimental Evaluation}
\label{sec:experimental_evaluation}

We assess the performance of DAG-MLP with $k$-NTs %
on a range of synthetic~\citep{abboud2021surprising, murphy2019relational} and real-world datasets~\citep{Craven1998LearningTE, McCallum2000AutomatingTC, Giles1998CiteSeerAA, Sen2008CollectiveCI, Datasets} (additional details can be found in Appendix~\ref{app:datasets}).\footnote{Our implementation is available (anonymized) at \url{https://anonymous.4open.science/r/k-RedundancyGNNs/} and will be uploaded to GitHub upon publication.}

\textbf{Experimental setup.} For synthetic datasets, we determine the number of layers in DAG-MLP based on the average graph diameter, ensuring effective aggregation during message propagation. %
The embeddings at each layer are obtained using readouts, concatenated, and then fed through two learnable linear layers for prediction.
In the evaluation of TUDataset, we adopt the 10-fold splits proposed by~\cite{Errica2020A}, allowing a grid search for optimal hyper-parameters on each dataset. The architecture for combined heights involves using each ``readout$_i$'' to extract the embeddings for each layer, with mean of the average-pooled embeddings being passed to a final MLP layer responsible for prediction (see Appendix~\ref{sec:dag_mlp_graph_class_architecture}). For the fixed single-height architecture, only the last readout is used, pooled, and passed to the final MLP layer. Further details on hyper-parameters can be found in Appendix~\ref{sec:hyperparameter}. %

\textbf{Graph classification.}
Table~\ref{tab:expressivity_results} presents the results on synthetic expressivity datasets. %
In line with our theoretical expectations, the experimental results support our hypothesis: NTs exhibit greater expressivity than GIN on a graph level. However, a formal theoretical proof of this observation remains a direction for future work.

\begin{table}[t]
\centering
\small
\caption{Average %
	classification accuracy for EXP-Class and CSL across $k$-folds (4-folds and 5-folds), and the number of indistinguishable pairs of graphs in EXP-Iso. Best results are highlighted in {\setlength{\fboxsep}{1pt}\colorbox{lightgray}{gray}}, best results from methods with polynomial time complexity are highlighted in \textbf{bold}.}
\label{tab:expressivity_results}
\begin{tabular}{lccc}
	\toprule
	\textbf{Model} & \textbf{EXP-Class} $\uparrow$ & \textbf{EXP-Iso} $\downarrow$ & \textbf{CSL} $\uparrow$ \\
	\midrule
	GIN~\citep{DBLP:journals/corr/abs-1810-00826} & \phantom{0}50.0 $\pm$ 0.0 & 600 & \phantom{0}10.0 $\pm$ 0.0\phantom{0} \\
	3WLGNN~\citep{maron2019provably} & \cellcolor{lightgray}\textbf{100.0} $\pm$ \textbf{0.0} & \cellcolor{lightgray}\textbf{0} & \phantom{0}97.8 $\pm$ 10.9 \\
	PathNN-$\mathcal{SP}^+$~\citep{michel_expressive_2023} & \cellcolor{lightgray}100.0 $\pm$ 0.0 & \cellcolor{lightgray}0 & \cellcolor{lightgray}100.0 $\pm$ 0.0 \\
	PathNN-$\mathcal{AP}$~\citep{michel_expressive_2023} & \cellcolor{lightgray}100.0 $\pm$ 0.0 & \cellcolor{lightgray}0 & \cellcolor{lightgray}100.0 $\pm$ 0.0 \\
	\midrule
	DAG-MLP ($0$-NTs) & \cellcolor{lightgray}\textbf{100.0} $\pm$ \textbf{0.0} & \cellcolor{lightgray}\textbf{0} & \cellcolor{lightgray}\textbf{100.0} $\pm$ \textbf{0.0}\phantom{0} \\
	DAG-MLP ($1$-NTs) & \cellcolor{lightgray}\textbf{100.0} $\pm$ \textbf{0.0} & \cellcolor{lightgray}\textbf{0} & \cellcolor{lightgray}\textbf{100.0} $\pm$ \cellcolor{lightgray}\textbf{0.0}\phantom{0} \\
	\bottomrule
\end{tabular}
\end{table}

In Table~\ref{tab:exp_case_study}, we investigate the impact of the parameter $k$ %
and the number of layers $l$ on the accuracy %
for the EXP-Class dataset.
Cases where $k > l$ can be disregarded, as the computation for NTs remains the same as when $k = l$. %
Empirically, $0$- and $1$-NTs yield the highest accuracy.
This observation aligns with our discussions on expressivity in Section~\ref{sec:1nt_expressivity}. The decrease in accuracy with increasing $k$ indicates that information redundancy leads to oversquashing. We investigate this theoretically in Appendix~\ref{app:oversquash}.

\begin{table}[tb]
\small
\caption{Average
	accuracy for DAG-MLP using 4-fold cross-validation on EXP-Class~\citep{abboud2021surprising}, evaluated with varying number of layers.}
\label{tab:exp_case_study}
\centering
\begin{tabular}{lcccccc}
	\toprule
	$k$-\textbf{NTs} & \textbf{1 layer} & \textbf{2 layers} & \textbf{3 layers} & \textbf{4 layers} & \textbf{5 layers} & \textbf{6 layers} \\
	\midrule
	$0$-NTs & 51.1 $\pm$ 1.6 & 57.5 $\pm$ 6.6 & 91.7 $\pm$ 11.6 & 99.7 $\pm$ 0.3 & \textbf{100.0} $\pm$ \textbf{0.0} & \textbf{100.0} $\pm$ \textbf{0.0} \\
	$1$-NTs & 50.1 $\pm$ 0.2 & 58.9 $\pm$ 4.6 & 59.4 $\pm$ 5.7 & 99.6 $\pm$ 0.5 & 99.9 $\pm$ 0.2 & \textbf{100.0} $\pm$ \textbf{0.0} \\
	$2$-NTs & - & 52.6 $\pm$ 3.4 & 54.9 $\pm$ 5.3 & 52.4 $\pm$ 3.8 & 97.6 $\pm$ 1.9 & \textbf{100.0} $\pm$ \textbf{0.0} \\
	$3$-NTs & - & - & 56.2 $\pm$ 5.7 & 51.1 $\pm$ 1.9 & 52.4 $\pm$ 4.1 & 87.1 $\pm$ 21.4 \\
	$4$-NTs & - & - & - & 50.1 $\pm$ 0.2 & 50.6 $\pm$ 1.0 & 50.4 $\pm$ 0.7 \\
	$5$-NTs & - & - & - & - & 50.4 $\pm$ 0.7 & 50.0 $\pm$ 0.0 \\
	$6$-NTs & - & - & - & - & - & 53.2 $\pm$ 5.2 \\
	\bottomrule
\end{tabular}
\end{table}

For TUDataset, we report the accuracy compared to related work in Table~\ref{tab:results_tudatasets}. We report only the best results for the different parameter combinations reported in~\citet{michel_expressive_2023}, and the best result for our different combine methods. Due to the high standard deviation across all methods, we present a statistical box plot for the accuracy based on three runs on the test set of 10-fold cross-validation in Appendix~\ref{app:additionalexp}. We group the methods by their time complexity. Note that, while PathNN performs well on ENZYMES and PROTEINS, the time complexity of this method is exponential. Therefore, we also highlight the best method with polynomial time complexity. For IMDB-B and IMDB-M, which have small diameters, we see that $k$-NTs outperform all other methods. For ENZYMES a variant of our approach achieves the best result among the approaches with non-exponential time complexity, and $k$-NTs lead to a significant improvement over GIN.

\begin{table*}[tb]
\small
\caption{
	Classification %
	accuracy ($\pm$ standard deviation) over 10-fold cross-validation on the datasets from TUDataset, taken from~\citet{michel_expressive_2023}. Best performance is highlighted in {\setlength{\fboxsep}{1pt}\colorbox{lightgray}{gray}}, best results from methods with polynomial time complexity are highlighted in \textbf{bold}. ``-'' denotes not applicable and ``NA'' means not available.}
\label{tab:results_tudatasets}
\centering
\begin{tabular}{clcccc}
	\toprule
	&\textbf{Model} & \textbf{IMDB-B} & \textbf{IMDB-M} & \textbf{ENZYMES} & \textbf{PROTEINS} \\
	\midrule
	\multirow{4}{0.05cm}{\rotatebox{90}{Linear}}&%
	GIN~\citep{DBLP:journals/corr/abs-1810-00826} & 71.2 $\pm$ 3.9 & 48.5 $\pm$ 3.3 & 59.6 $\pm$ 4.5 &  73.3 $\pm$ 4.0 \\
	&GAT~\citep{velivckovic2018graph} & 69.2 $\pm$ 4.8 & 48.2 $\pm$ 4.9 & 49.5 $\pm$ 8.9 & 70.9 $\pm$ 2.7 \\
	&SPN ($l=1$)~\citep{AbboudDC22} & NA & NA & 67.5 $\pm$ 5.5 & 71.0 $\pm$ 3.7 \\
	&SPN ($l=5$)~\citep{AbboudDC22} & NA & NA & 69.4 $\pm$ 6.2 & \textbf{74.2} $\pm$ \textbf{2.7} \\
	\midrule
	\multirow{3}{0.05cm}{\rotatebox{90}{Exp}}&%
	PathNet~\citep{sun2022beyond} & 70.4 $\pm$ 3.8 & 49.1 $\pm$ 3.6 & 69.3 $\pm$ 5.4 & 70.5 $\pm$ 3.9 \\
	&PathNN-$\mathcal{P\phantom{S}}\phantom{^+}$~\citep{michel_expressive_2023} & 72.6 $\pm$ 3.3 & 50.8 $\pm$ 4.5 & \cellcolor{lightgray}73.0 $\pm$ 5.2 &\cellcolor{lightgray}75.2 $\pm$ 3.9\\
	&PathNN-$\mathcal{SP}^+$~\citep{michel_expressive_2023} & - & - & 70.4 $\pm$ 3.1 & 73.2 $\pm$ 3.3 \\
	\midrule
	\multirow{2}{0.05cm}{\rotatebox{90}{Ours}}
 &DAG-MLP ($0$-NTs) & \cellcolor{lightgray}\textbf{72.9} $\pm$ \textbf{5.0} & 50.2 $\pm$ 3.2 & 67.9 $\pm$ 5.3  & 70.1 $\pm$ 1.7 \\
	&DAG-MLP ($1$-NTs)& 72.4 $\pm$ 3.8 & \cellcolor{lightgray}\textbf{51.3} $\pm$ \textbf{4.4} & \textbf{70.6} $\pm$ \textbf{5.5} & 70.2 $\pm$ 3.4 \\
	\bottomrule
\end{tabular}
\end{table*}

\textbf{Node classification.}
We investigate the performance of our approach on node classification datasets. These datasets differ regarding their homophily ratio, i.e., the fraction of edges in a graph that connect vertices with the same class label~\cite{DBLP:conf/nips/ZhuYZHAK20}. Heterophily tasks are particularly challenging for standard GNNs~\cite{DBLP:conf/nips/ZhuYZHAK20} as they require capturing the structure of neighborhoods instead of “averaging” over the neighboring features. In Table~\ref{tab:node_classification} we present results from~\cite{DBLP:conf/cikm/GiraldoSBM23} including the state-of-the-art graph rewiring technique SJLR combined with SGC and GCN, which performs best in the evaluation. We also performed experiments with GIN and DAG-MLP using the same data splits as~\cite{DBLP:conf/cikm/GiraldoSBM23} to ensure a fair comparison. We report the best results for $l$ layers with $l \in \{2,3,4\}$ and four different combine methods for GIN and DAG-MLP.

\begin{table}[tb]
    \centering
    \caption{Accuracy and standard deviation on node classification tasks (GCN, SJLR-GCN, SGC and SJLR-GCN taken from~\cite{DBLP:conf/cikm/GiraldoSBM23}). %
    }
    \label{tab:node_classification}
    \scalebox{0.85}{
    \begin{tabular}{cccccccc}
    \toprule
      & \textbf{Model}& \textbf{Texas} &  \textbf{Wisconsin} &\textbf{Cornell}&\textbf{Cora} &  \textbf{CiteSeer} &\textbf{PubMed}\\
       &  Homophily ratio & 0.11  & 0.21&0.3 &0.8& 0.74  & 0.8    \\
         \midrule
       &  GCN & 58.05 $\pm$ 0.9 & 52.10 $\pm$ 0.9 &  67.34 $\pm$ 1.5&  81.81 $\pm$ 0.2 &  68.35 $\pm$ 0.3 &  78.25 $\pm$ 0.3 \\
       &  SJLR-GCN & 60.13 $\pm$ 0.8&  55.16 $\pm$ 0.9 &  71.75 $\pm$ 1.5 &  \textbf{81.95 $\pm$ 0.2} & \textbf{69.50 $\pm$ 0.3} &  \textbf{78.60 $\pm$ 0.3} \\
       &  SGC &  56.69 $\pm$ 1.7&  47.90 $\pm$ 1.7 & 53.40 $\pm$ 2.1&  76.90 $\pm$ 1.3&  67.45 $\pm$ 0.8& 71.79 $\pm$ 2.1 \\
       &  SJLR-SGC & 58.40 $\pm$ 1.4&  55.42 $\pm$ 0.9 &  67.37 $\pm$ 1.6 &  81.24 $\pm$ 0.7& 68.39 $\pm$ 0.6 &  76.28 $\pm$ 0.9 \\
         \midrule
         &
         GIN   &  73.78 $\pm$ 6.0  &   71.76 $\pm$ 5.1  &  60.81 $\pm$ 8.5 &  76.76 $\pm$ 1.4&   64.49 $\pm$ 1.5 &  76.46 $\pm$ 1.1 \\
         &DAG-MLP ($0$-NTs)&  \textbf{85.68 $\pm$ 4.8}  &   81.35 $\pm$ 4.1  &  79.02 $\pm$ 6.8 & 74.01 $\pm$ 2.0&   60.55 $\pm$ 3.6  &  75.33 $\pm$ 1.1 \\
         &DAG-MLP ($1$-NTs)&  80.54 $\pm$ 6.0  &   \textbf{81.62 $\pm$ 3.4}  &  \textbf{79.41 $\pm$ 4.6} & 74.54 $\pm$ 1.4&  61.09 $\pm$ 1.5 &  75.53 $\pm$ 1.1 \\
         \bottomrule
    \end{tabular}
    }
\end{table}

As observed, DAG-MLP outperforms GIN on the heterophily datasets (those with low homophily ratio), while GIN performs better on homophily ones. These results indicate that neighborhood trees can capture the relevant neighborhood structure more accurately than unfolding trees used by GIN. Additionally, our method outperforms SJLR on the two heterophily datasets Texas and Wisconsin by a large margin.

\section{Conclusion}
We introduce a neural tree canonization technique and combine it with neighborhood trees, which are pruned versions of unfolding trees used by standard MPNNs. By merging trees in a DAG, we create compact representations that serve as the foundation for our neural architecture termed DAG-MLP. %
It inherits the advantageous properties of the GIN architecture, while being more expressive than $1$-WL on many graphs. Notably, our method is only less expressive on node level for specific examples.
Our work contributes general techniques for constructing compact computation DAGs for tree structures that encode node neighborhoods. This exploration reveals a complex interplay between information redundancy, computational redundancy, and expressivity. The delicate balance of these factors is an avenue for future work.

\section*{Acknowledgments}
We would like to thank Christian Permann for his contribution to the conception of neighborhood trees and their efficient generation.
This work was supported by the Vienna Science and Technology Fund (WWTF) [10.47379/VRG19009]. 
The computational results presented have been achieved in part using the Vienna Scientific Cluster (VSC).

\section*{Author Contributions}
NK devised the project and the main conceptual ideas.
FB made significant contributions to the conception of $k$-redundant neighborhood trees and their efficient generation.
NK and FB jointly developed the methods for redundancy removal presented in Sections~\ref{sec:canon},~\ref{sec:comp_redundancy}.
FB developed Theorem~\ref{thm:expr}, and implemented the $k$-redundant neighborhood trees, merge DAGs, and algorithmic components of the implementation.
NK, SM, and FB collaborated on the development of the DAG-MLP architecture.
SM implemented DAG-MLP, conducted the experimental evaluation, and wrote parts of the corresponding sections in the manuscript. SM extended support to directed graphs and graphs with edge attributes, implemented the learning pipeline, wrapped the DAG to be used within the learning pipeline, and configured the necessary environments to reproduce results.
FB and NK jointly drafted the manuscript with input from all authors.
FB revised the manuscript with feedback from reviewers.
NK, WG, and JL supervised the project.
All authors provided critical feedback, participated in discussions, contributed to the interpretation of the results, and approved the final manuscript.

\bibliography{lit.bib}

\newpage
\appendix

\section{Comparison to RFGNN and TPTs}\label{sec:rfgnn}
In RFGNN~\citep{NEURIPS2022_1bd6f176}, truncated epath trees (TPTs) are introduced to represent the information flow with the goal of reducing redundancy. While the motivation aligns with our approach, there are significant differences between TPTs and $k$-redundant neighborhood trees and the computational properties of the techniques. In RFGNN the focus is solely on reducing redundancy in information flow, not computation. Additionally, the definition of TPTs allows for much more redundancy than that of $k$-NTs. We first introduce the concepts used by~\citet{NEURIPS2022_1bd6f176}, and then discuss differences and disadvantages in detail.

An \emph{epath} is defined as a path with no repeated vertices, except the starting vertex, which is allowed to be the ending vertex if the length of the epath is larger than $2$.
\begin{definition}[Truncated ePath Tree~\citep{NEURIPS2022_1bd6f176}]\label{def:tpt}
Given graph $G$ and $v \in V(G)$, the TPT $TPT^h_{G,v}$ with height $h$ is an epath search tree obtained by running a BFS from $v$, where all epaths of length up to $h$ are accessed.
\end{definition}

Firstly, the definition of TPTs allows for vertices to redundantly appear multiple times in a tree. If a vertex appeared at depth $1$, for example, it can still appear elsewhere in the TPT, but not as its own descendant.  
In TPTs, parts of paths that differ will be repeated, without the ability to compress them. In contrast, neighborhood trees allow for compressed representations, as demonstrated in Appendix~\ref{sec:compacttrees}.
\citet{NEURIPS2022_1bd6f176}, Lemma~7 claims that TPTs are more expressive than unfolding trees of the same height by providing two example graphs, which can also be distinguished by $1$-NTs of the same height, cf.~Theorem~\ref{thm:expr}. \citet{NEURIPS2022_1bd6f176}, Lemma~6 states, that for any two nodes, if the unfolding trees of height $k$ differ, the TPTs of height $k$ differ as well, however, this is not true. Figure~\ref{ex:ce_rfgnn} shows two nodes that can be distinguished by their unfolding trees (of height $\geq 4$), but not by their TPTs. The problem is the same as for NTs - they have a maximum height, whereas unfolding trees will grow indefinitely. Hence, TPTs are not strictly more expressive than unfolding trees on node level, for the same reason that $k$-NTs are not. Unlike NTs, TPTs have several other disadvantages.

The size and running time complexity of RFGNN are very restrictive. While the term BFS in the definition implies linear running time, the BFS has to be modified, leading to exponential running time.
The compression of TPTs (or even the forest of TPTs for the vertices of a graph) is not discussed in the publication, making preprocessing and computation much more time-consuming.
In the experimental evaluation, only TPTs up to height $3$ are used due to resource-intensity, indicating that the full expressivity of TPTs cannot be utilized in practice. This limitation is also reflected in the experimental results in~\citep{NEURIPS2022_1bd6f176}, where the outcomes using RFGNN are only marginally better.

\begin{figure}[tb]
\centering
\begin{subfigure}[b]{0.25\linewidth}\centering
	\includegraphics[width=0.4\textwidth]{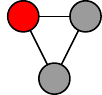}
	\subcaption{$G_1$}\label{fig:ce_rfgnn_g1}
\end{subfigure}
\begin{subfigure}[b]{0.25\linewidth}\centering
	\includegraphics[width=0.4\textwidth]{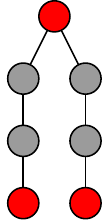}
	\subcaption{$G_2$}\label{fig:ce_rfgnn_g2}
\end{subfigure}
\begin{subfigure}[b]{0.25\linewidth}\centering
	\includegraphics[width=0.4\textwidth]{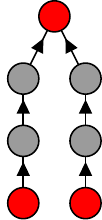}
	\subcaption{TPTs of red vertices}\label{fig:ce_tpt}
\end{subfigure}
\caption{Two graphs (\subref{fig:ce_rfgnn_g1}), (\subref{fig:ce_rfgnn_g2}) in which the red vertices can be distinguished by unfolding trees, but not by TPTs. Figure~(\subref{fig:ce_tpt}) shows the TPTs of the red vertices (the red vertex in the middle of $G_2$), which are the same.}
\label{ex:ce_rfgnn}
\end{figure}

With our proposed approach, we address not only redundancy in information flow using $k$-NTs, but also remove redundancy in computation by incorporating merge DAGs. This strategy enables our approach to reach its full expressive potential in practice while maintaining a reasonable running time.

\section{Theoretical Analysis of Oversquashing in Comparison with RFGNN}
\label{app:oversquash}
Several authors~\cite{NEURIPS2022_1bd6f176,DBLP:conf/icml/XuLTSKJ18,ToppingGC0B22,DBLP:conf/icml/GiovanniGBLLB23,Black23} developed and refined techniques to measure the influence of a vertex $v$ with an initial vertex feature $\mathbf{x}_v$ on the output $\mathbf{h}_u^{(k)}$ of a vertex $u$ after layer $k$ by the Jacobian $\partial \mathbf{h}_u^{(k)}/ \partial \mathbf{x}_v$. Following~\citet[Lemma 3]{NEURIPS2022_1bd6f176}, the relative influence of a node $v$ on a node $u$ in an MPNN is 
$$I(v,u)=
\mathbb{E}\left(\frac{\partial \mathbf{h}_u^{(k)}/ \partial \mathbf{x}_v}{\sum_{w\in V} \partial \mathbf{h}_u^{(k)}/ \partial \mathbf{x}_w}\right)=
\frac{[\hat{A}^k]_{u,v}}{\sum_{w\in V}[\hat{A}^k]_{u,w}},$$
where $\hat{A}=A+I$ is the adjacency matrix of the graph $G$ with added self-loops. Note that $[\hat{A}^k]_{u,v}$ is the number of walks of length $k$ from $u$ to $v$ (and vice versa) in $G$ with added self-loops.
Oversquashing occurs when $I(v,u)$ becomes small, indicating that only a small fraction of walks of length up to $k$ ending at $u$ start at $v$.

This idea can easily be linked to the concept of unfolding trees underlying our work. Consider the unfolding tree $F^u_k$ of vertex $u$ with height $k$. It follows from its construction that there is a bijection between walks of length at most $k$ ending at $u$ in $G$ and paths in $F^u_k$ from some node to the root (see~\cite{DBLP:conf/nips/Kriege22} for details on unfolding trees and walks). Therefore, pruning the unfolding tree has an effect on walk counts and, thus, on the relative influence. Consider the example in Figure~\ref{fig:unfoldingtree} and let $u$ be the red vertex (upper left) and $v$ the yellow vertex (lower right). We obtain a relative influence of $I_{\text{MPNN}}(v,u)= \frac{1}{8}$ for unfolding trees, and $I_{0\text{NT}}(v,u)= \frac{1}{4}$ for $0$-NTs, showing that NTs have the potential to reduce oversquashing.

We formally show that our method is less susceptible to oversquashing than MPNNs and RFGNN~\cite{NEURIPS2022_1bd6f176}. Consider a vertex $v$ and a vertex $u$ with shortest-path distance of $k$. To pass information from $v$ to $u$, at least $k$ layers are required. Comparing the unfolding tree (MPNN), the $0$- and $1$-NT (our approach) and the TPT (RFGNN), all of height $k$, reveals that the vertex $v$ occurs in the last level only, i.e., as a leaf of the tree, and the number of occurrences of $v$ is equal in all trees, since all walks and simple paths of length $k$ reaching $v$ are shortest paths. Hence, the numerator of the relative influence is equal for all methods. However, since $0$- and $1$-NTs are subtrees of unfolding trees, and $0$-NTs/$1$-NTs are subtrees of TPTs (they contain only shortest paths/some simple paths, instead of all simple paths), the total number of nodes in the trees, i.e., walks contributing to the denominator of the relevant information can be compared, obtaining 
$$I_{\text{MPNN}}(v,u)\leq I_{\text{TPT}}(v,u) \leq I_{1\text{NT}}(v,u) \leq I_{0\text{NT}}(v,u).$$
This theoretical analysis shows that our proposed method offers advantages in mitigating oversquashing, leveraging the formalization developed in recent papers. Additionally, it establishes a theoretical connection between the proposed approach and RFGNN, highlighting that our method more effectively addresses the oversquashing problem.

\section{The AHU Algorithm}\label{sec:ahu}
Aho, Hopcroft and Ullman describe a linear-time algorithm for deciding whether two rooted unordered trees are isomorphic~\cite[Section 3.2]{AhoHU74}. The algorithm forms the basis for our neural tree canonization technique, as discussed in Section~\ref{sec:canon}, and serves as a fundamental subroutine for combining trees into a single merge DAG, as detailed in Appendix~\ref{sec:mergetrees_appendix}. Here, we give a complete description of the original algorithm, its extension to trees with node labels or features, and the required modification for tree canonization.

In its original version, the algorithm solves the subtree isomorphism problem for two rooted unordered unlabeled trees $T_1$ and $T_2$. Algorithm~\ref{alg:ahu} shows the pseudocode of the algorithm.\footnote{For clarity of presentation, we adapted and simplified the textual description of the textbook~\citep{AhoHU74}. In contrast to the original description, our algorithm operates on the disjoint union of both trees instead of applying the same operations to $T_1$ and $T_2$ individually.}
First, the nodes in the disjoint union of the input trees $T_1 \cup T_2$ are partitioned into levels according to their depth distinguishing leaves and internal nodes, see Figure~\ref{fig:ahu}. Note that the levels are numbered in reverse order of depth, i.e., for a node $v$ on level $i$ the equality $\depth(v)=\height(T)-i$ holds. The lists $\mathcal{L}^*_i$ and $\mathcal{L}_i$ contain all leaves and internal nodes, respectively, on level $i$.
The labels $\ahu$ of the leaves are set to 0 and the tree is processed in bottom-up-fashion. In iteration $i$ of the for-loop, the labels of all nodes $\mathcal{L}_k$ for all $k<i$ have been computed and $\mathcal{L}_k$ is sorted according to them. Note that $\mathcal{L}^*_i$ contains only nodes $v$ with $\ahu(v)=0$ for all $0\leq i\leq \height(T)$. Tuples $\ahut(v)$ are generated for the nodes $v$ in $\mathcal{L}_i$ by iterating over $\mathcal{L}^*_{i-1}$ and then $\mathcal{L}_{i-1}$ appending the label of the current node to the tuple of its parent. Each tuple contains an integer label for each child and is in ascending order. In the next step, the nodes in $\mathcal{L}_i$ are sorted according to the tuples using radix sort. Then, the \textsc{Relabel} function assigns new integers $\ahu(v)$ to all nodes $v$ in $\mathcal{L}_i$ based on their tuples. Since $\mathcal{L}_i$ is sorted, all nodes with the same label form a contiguous sub-list. New integers are computed by scanning the list assigning 1 to the first entry and increasing the integer whenever the current tuple differs from the previous one. Using this approach, the \textsc{Relabel} function computes an injection between tuples and integers appearing for the nodes in $\mathcal{L}_i$. Two trees are isomorphic if and only if their nodes yield the same multiset of labels on all levels. Figure~\ref{fig:ahu} shows an example of two trees that are identified as isomorphic. The algorithm can be implemented in linear time by applying \textsc{RadixSort} to sort tuples of labels from a bounded range.

As noted by~\citet{AhoHU74}, the algorithm can be extended to trees with initial integer labels in range 1 to $n$ with $n=O(|V(T)|)$, by including the label of a node as the first element in its tuple. In this case, the \textsc{Relabel} function assigns integers that were not used as initial labels and the leaves in $\mathcal{L}^*_i$ have to be sorted according to their label after initialization. The overall running time remains linear.\footnote{A similar technique including level-wise processing, creation of tuples sorted by radix sort and relabeling has been proposed by~\citet[Section 2.1]{JMLR:v12:shervashidze11a} in the context of the Weisfeiler-Leman kernel to achieve a linear running time.} If the labels are not integers from a bounded range, e.g., continuous values, an initial mapping to integers is required, which can be realized by comparison-based sorting in $O(n \log n)$.

In order to generalize the method to tree canonization, it is no longer sufficient that the relabeling function is injective for the tuples appearing on each level, but it has to be injective for all possible tuples that can occur in \emph{any} tree. We discuss this situation in Section~\ref{sec:canon} and propose a learnable function with this property.

\begin{algorithm}[tb]
\caption{AHU algorithm for tree isomorphism}\label{alg:ahu}
\begin{algorithmic}
	\Function{Relabel}{nodes $\mathcal{L}$, labels $\ahut$, $\ahu$}
	\Comment{Replaces tuple labels $\ahut(v)$ by integer labels $\ahu(v)$ for the nodes $v$ in the list $\mathcal{L}=(l_1,l_2,\dots,l_N)$ sorted according to $\ahut$.}
	\State $\textsf{prev} \gets \ahut(l_1)$ 
	\State $k \gets 1$
	\For{$i \gets 1$ \textbf{to} $N$}
	\If{$\ahut(l_i) = \textsf{prev}$} 
	\State $\ahu(l_i) \gets k$
	\Else{}
	\State $k\gets k+1$  \Comment{Next integer label}
	\State $\ahu(l_i) \gets k$
	\EndIf
	\State $\textsf{prev} \gets \ahut(l_i)$ 
	\EndFor
	\Return $\ahu$
	\EndFunction
	\State
	\Function{TreeIsmorphism}{$T_1,T_2$}
	\State $T\gets T_1 \cup T_2$
	\State $H\gets \height(T)$
	\For{$i \gets 0$ \textbf{to} $H$}
	\State $\mathcal{L}^*_i \gets \{ v \in V(T) \mid \depth(v)=H-i \wedge \children(v)=\emptyset\}$ \Comment{Leaves with the same depth}
	\State $\mathcal{L}_i \gets \{ v \in V(T) \mid \depth(v)=H-i \wedge \children(v)\neq\emptyset\}$ \Comment{Non-leaves with the same depth}
	\EndFor
	\ForEach{ leaf $v \in V(T)$}
	\State $\ahu(v)\gets 0$ \Comment{Initialize labels for leaves}
	\EndFor
	\For{$i \gets 1$ \textbf{to} $H$}
	\ForEach{$l$ in ordered list $\mathcal{L}^*_{i-1}+\mathcal{L}_{i-1}$} \Comment{Iterate over concatenated ordered list}
	\State Append $\ahu(l)$ to the tuple $\ahut(p(l))$ \Comment{Pass label of previous level upwards}
	\EndFor
	\State $\mathcal{L}_{i}\gets$\textsc{RadixSort}($\mathcal{L}_{i}$, $\ahut$) \Comment{Sort list according to their tuples}
	\State $\ahu \gets$\textsc{Relabel}($\mathcal{L}_{i}$, $\ahu$, $\ahut$)
	\If{$\multiset{\ahu(v) \mid  (\mathcal{L}^*_i \cup \mathcal{L}_i) \cap V(T_1)}\neq\multiset{\ahu(v) \mid  (\mathcal{L}^*_i \cup \mathcal{L}_i) \cap V(T_2)}$ } %
	\Return \texttt{false}
	\EndIf
	\EndFor
	\Return \texttt{true}
	\EndFunction
\end{algorithmic}
\end{algorithm}

\begin{figure}
\begin{tikzpicture}[
level distance=1.3cm,
level 1/.style={sibling distance=2.2cm},
level 2/.style={sibling distance=.8cm},
level 3/.style={sibling distance=.6cm},
edge from parent/.style={draw},
vertex/.style = {circle,draw,fill=white,inner sep=1.5, outer sep=0},
label distance=-2
]
\draw[black!50!white,dashed] (-3, 0) node[left]{$3$} node[yshift=15, xshift=-5]{Level} -- (10, 0);
\draw[black!50!white,dashed] (-3,-1.3) node[left]{$2$} -- (10,-1.3);
\draw[black!50!white,dashed] (-3,-2.6) node[left]{$1$} -- (10,-2.6);
\draw[black!50!white,dashed] (-3,-3.9) node[left]{$0$} -- (10,-3.9);

\node [vertex, label={above right:$(1,2,3)$}] (r){$1$}
child {node [vertex, label={below right:$(2)$}] (a) {$1$}
	child {node [vertex, label={below right:$(0,0)$}] (b) {$2$} 
		child {node [vertex] (c) {$0$}}
		child {node [vertex] (d) {$0$}}
	}
}
child {node [vertex, label={below right:\,\,\,$(0,0,2)$}] (e) {$3$}
	child {node [vertex] (f) {$0$}}
	child {node [vertex, label={below right:$(0,0)$}] (g) {$2$}
		child {node [vertex] (h) {$0$}}
		child {node [vertex] (i) {$0$}}
	}
	child {node [vertex] (j) {$0$}}
}
child {node [vertex, label={below right:$(0,1)$}] (k) {$2$}
	child {node [vertex] (l) {$0$}}
	child {node [vertex, label={below right:$(0)$}] (m) {$1$}
		child {node [vertex] (n) {$0$}}
	}
};

\node [vertex,xshift=7cm,label={above right:$(1,2,3)$}] (R){$1$}
child {node [vertex,label={below right:$(0,1)$}] (K) {$2$}
	child {node [vertex, label={below right:$(0)$}] (M) {$1$}
		child {node [vertex] (N) {$0$}}
	}
	child {node [vertex] (L) {$0$}}
}
child {node [vertex, label={below right:\,\,\,$(0,0,2)$}] (E) {$3$}
	child {node [vertex, label={below right:$(0,0)$}] (G) {$2$}
		child {node [vertex] (H) {$0$}}
		child {node [vertex] (I) {$0$}}
	}
	child {node [vertex] (F) {$0$}}
	child {node [vertex] (J) {$0$}}
}
child {node [vertex, label={below right:$(2)$}] (A) {$1$}
	child {node [vertex, label={below right:$(0,0)$}] (B) {$2$}
		child {node [vertex] (C) {$0$}}
		child {node [vertex] (D) {$0$}}
	}
};
\end{tikzpicture}
\caption{Two isomorphic trees $T_1$ (left) and $T_2$ (right) and the labels $\ahu$ (inside each node) and $\ahut$ (right of each node) computed by the AHU algorithm.}
\label{fig:ahu}
\end{figure}

\section{Building Compact Trees}\label{sec:compacttrees}
Since unfolding trees can grow exponentially in size and our goal is to avoid redundant computation, we do not build unfolding trees and $k$-NTs explicitly. Rather, we build DAGs that represent them, corresponding to the merge DAG of only that tree using $L=\phi$. This way, the $k$-NTs can be generated by a simple, slightly modified BFS algorithm, and the size of $k$-NTs is in $O(|E(G)|\cdot(k+1))$, which means it is linear in the size of the input graph $G$.

\section{Merging Trees -- Algorithm}\label{sec:mergetrees_appendix}
\begin{algorithm}[tb]
\caption{Merging trees}\label{alg:merge}
\begin{algorithmic}
	\Function{merge}{set of trees $\mathcal{T}$, labeling $L$}\Comment{merges $\mathcal{T}$ into a DAG $D$}
	\State $D \gets $ empty DAG  \Comment{start with empty DAG}
	\State initialize $D.can\_map$ as an empty map \Comment{maps canonization to node in DAG}
	\ForEach{$T \in \mathcal{T}$}
	\State compute canonization $can(v)$ for $v \in V(T)$  under $L$
	\State \textsc{add}($D$, $T$, $r(T)$, $L$)\Comment{add tree, starting at root}
	\EndFor
	\Return $D$
	\EndFunction
	\State
	\Function{add}{DAG $D$, tree $T$, vertex $v$, labeling $L$} \Comment{adds substructure rooted at $v \in V(T)$ to $D$}
	\If{$can(v) \in D.can\_map$} \Comment{node (and substructure) already present in $D$}
	\Return 
	\EndIf
	\ForEach{$c \in \children(v) $}  \Comment{add all children first (if necessary)}
	\State \textsc{add}($D$, $T$, $c$, $L$)
	\EndFor
	\State add new node $v_2$ with $L(v_2) = L(v)$ to $D$ \Comment{add new node}
	\State set $can(v_2) = can(v)$ and $D.can\_map(can(v_2)) = v_2$
	\ForEach{$c \in \children(v) $} 
	\If{edge exists from $D.can\_map(can(c))$ to $v_2$} 
	\State increase multiplicity of edge by $1$ \Comment{a sibling had the same canonization}
	\Else{}
	\State add edge from $D.can\_map(can(c))$ to $v_2$ \Comment{add edges from children to new node}
	\EndIf
	\EndFor
	\EndFunction
\end{algorithmic}
\end{algorithm}

Algorithm~\ref{alg:merge} describes how to merge a set of trees $\{T_1, \dots, T_n\}$ into a DAG under a labeling function $L\colon \bigcup_{i \in \{1, \dots, n\}} V(T_i) \to \mathcal{O}$, where $\mathcal{O}$ is some arbitrary labeling.
All substructures that are isomorphic under $L$ are merged. For that, the canonization of all vertices is computed first. Then each tree is merged to the DAG separately:
Starting at the root $r(T)$ of the tree that is added, if a node with the same canonization as $r(T)$ exists in the DAG, nothing needs to be done. Otherwise,the subtrees rooted at the children of $r(T)$ are added first (using the same procedure as for $r(T)$), and then a new node for $r(T)$ is added along with edges to the nodes in the DAG that have the same canonization as the children of $r(T)$. Note that, if some children have the same canonization, in this step multiedges can occur. The algorithm can easily be extended to merge DAGs by iterating over all roots in $\textsc{merge}$ and adding them to the DAG.
The running time of the algorithm depends on the canonization, which can be done in time linear in the numbers of nodes (see Appendix~\ref{sec:ahu}), and the time needed to add the trees to the DAG. Since we add each node at most once, and can check whether a canonization is already present in the DAG in constant time, this also only needs time linear in the number of tree nodes.

\section{DAG-MLP Architecture for Graph Classification }
\label{sec:dag_mlp_graph_class_architecture}

\begin{figure}
\centering
\includegraphics[width=\textwidth]{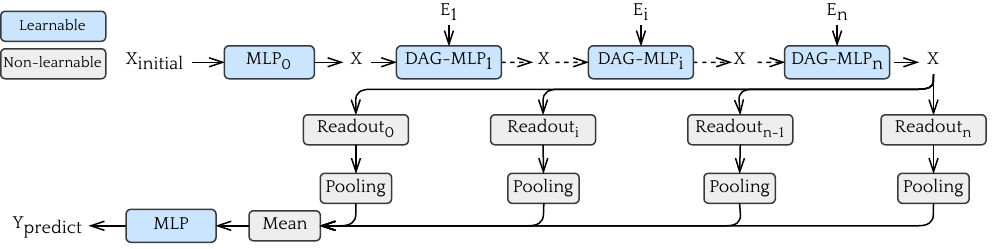}
\caption{DAG-MLP architecture with $n$ layers for graph-level prediction tasks.}
\label{fig:dag_mlp_graph_class_arch}
\end{figure}

Figure~\ref{fig:dag_mlp_graph_class_arch} shows an example of the architecture when using unfolding or neighborhood trees of height up to $n$ for graph classification requiring $n$ DAG-MLP layers.
The vertex features are initially transformed into embedding with fixed dimension using MLP$_0$.
Messages are then propagated using the DAG from height $0$ to $1$ ($\mathcal{E}_1$), which corresponds to layer $1$.
This process is repeated for $n$ layers, where the $i$th step computes embeddings for nodes of height $i$ in the DAG.
After $n$ layers, all node embeddings in the DAG ($X$) have been updated.
Using readouts, we extract the embeddings of each vertex from $k$-NTs of different heights within the DAG.
These extracted embeddings correspond to the embeddings of different layers.
A pooling operation is then applied to the output of each layer, and the pooled outputs are averaged.
These averaged outputs are passed through a final MLP, transforming them into probabilities for class prediction.

\section{Additional Experiments}
\label{app:additionalexp}
Following the same experimental setup as in Table~\ref{tab:exp_case_study}, Table~\ref{tab:csl_case_study} shows the accuracy with varying parameters $k$ and $l$. Since the expressive capabilities are the same as those of GIN when the number of layers $l$ equals the redundancy parameter $k$, all results with $l=k$ are not better than guessing.

\begin{table}[tb]
\caption{Average accuracy of DAG-MLP using 5-fold cross-validation on CSL~\citep{murphy2019relational}, evaluated with varying parameters $k$ and $l$.}
\label{tab:csl_case_study}
\centering
\begin{tabular}{lcccccc}
	\toprule
	$k$-\textbf{NTs} & \textbf{1 layer} & \textbf{2 layers} & \textbf{3 layers} & \textbf{4 layers} & \textbf{5 layers} & \textbf{6 layers} \\
	\midrule
	$0$-NTs  & $10.0 \pm 0.0$ & $20.0 \pm 0.0$ & $40.0 \pm 0.0$ & $70.0 \pm 0.0$  & $84.0 \pm 4.9$  & \textbf{100.0} $\pm$ 0.0 \\
	$1$-NTs  & $10.0 \pm 0.0$ & $10.0 \pm 0.0$ & $30.0 \pm 0.0$ & $48.0 \pm 4.0$  & $78.0 \pm 9.8$  & \textbf{100.0} $\pm$ 0.0 \\
	$2$-NTs  & -            & $10.0 \pm 0.0$ & $16.0 \pm 4.9$ & $20.0 \pm 12.6$ & $50.0 \pm 8.9$  & $80.0 \pm 12.6$ \\
	$3$-NTs  & -            & -             & $10.0 \pm 0.0$ & $10.0 \pm 0.0$  & $20.0 \pm 12.6$ & $38.0 \pm 14.7$ \\
	$4$-NTs  & -            & -             & -             & $10.0 \pm 0.0$  & $16.0 \pm 8.0$  & $34.0 \pm 12.0$ \\
	$5$-NTs  & -            & -             & -             & -               & $10.0 \pm 0.0$  & $10.0 \pm 0.0$  \\
	$6$-NTs  & -            & -             & -             & -               & -               & $10.0 \pm 0.0$  \\
	\bottomrule
\end{tabular}
\end{table}

Table~\ref{tab:mutag_different_heights_vs_fixed_heights} shows a comparison of $0$- or $1$-NTs, with combined heights and fixed single-height, for 10-fold cross-validation on MUTAG. The results as well as those shown in Table~\ref{tab:comparison_combine} indicate that using multiple different tree heights does not improve the generalization capabilities of the model.

\begin{table}[tb]
\scriptsize
\renewcommand{\arraystretch}{1.0}
\centering
\caption{Classification accuracy for 10-folds ($\pm$ standard deviation) on MUTAG comparing DAG-MLP that combines layers of different heights to DAG-MLP that only uses layers at a fixed height.}
\label{tab:mutag_different_heights_vs_fixed_heights}
\begin{tabular}{lccc|cccc}
	\toprule
	\multirow{2}{*}{\textbf{$k$-NTs}} & \multicolumn{3}{c|}{\textbf{Combine Heights}} & \multicolumn{3}{c}{\textbf{Fixed Single-Height}} \\
	& \textbf{1 layer} & \textbf{2 layers} & \textbf{3 layers} & \textbf{1 layer} & \textbf{2 layers} & \textbf{3 layers} \\
	\midrule
	$0$-NTs & 84.6 $\pm$ 6.2 & 86.7 $\pm$ 5.3 & 86.9 $\pm$ 6.0 & 85.3 $\pm$ 6.3 & 89.0 $\pm$ 4.7 & 87.2 $\pm$ 5.1 \\
	$1$-NTs & 84.9 $\pm$ 6.0 & 83.3 $\pm$ 7.3 & 88.6 $\pm$ 6.7 & 85.8 $\pm$ 6.0 & 88.8 $\pm$ 4.4 & 90.4 $\pm$ 5.1 \\
	\bottomrule
\end{tabular}
\end{table}

Figure~\ref{fig:tudataset_statistics} shows box plot charts for the accuracy obtained in Table~\ref{tab:results_tudatasets}. Due to the use of 10-fold cross-validation and the random initialization of the MLPs, the results tend to have high variance. For all datasets, the accuracy of DAG-MLP is statistically within the same boundaries as those of the best related methods reported in Table~\ref{tab:results_tudatasets}.

\begin{figure}
\centering
\subfloat[\centering Fixed Single-Height]{{\includegraphics[width=0.45\linewidth]{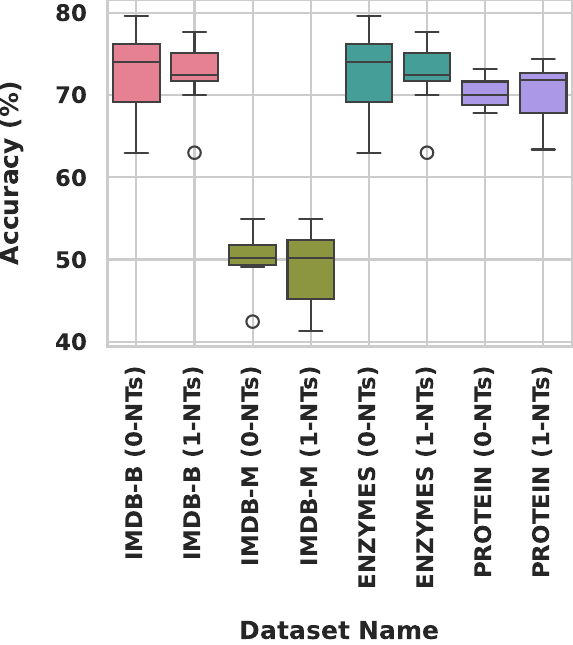} }}%
\qquad
\subfloat[\centering Combine Heights]{{\includegraphics[width=0.45\linewidth]{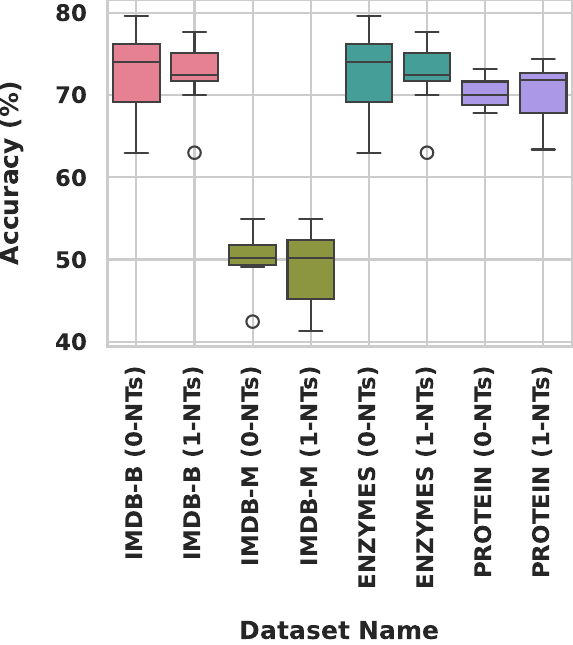} }}%
\caption{Graph classification test accuracy box plot over three runs of 10-fold cross-validation of the DAG-MLP on the datasets from the TUDataset. %
}
\label{fig:tudataset_statistics}
\end{figure}

\begin{table}[h]
\centering
\caption{Performance comparison of GIN and DAG-MLP architectures with different combination strategies between layers across datasets. The best result for each method on each dataset is marked in \textbf{bold}.}
\label{tab:comparison_combine}
\begin{tabular}{lccc}
\toprule
\textbf{Model} & \textbf{Texas} & \textbf{Wisconsin} & \textbf{Cornell} \\
\midrule
GIN (l=2) - Without Combine & 61.89 $\pm$ 7.0 & 57.65 $\pm$ 5.6 & 46.22 $\pm$ 6.0 \\
GIN (l=2) - Sum Combine & 69.46 $\pm$ 7.3 & 68.04 $\pm$ 6.0 & 58.11 $\pm$ 5.8 \\
GIN (l=2) - Mean Combine & 70.81 $\pm$ 7.1 & 68.82 $\pm$ 5.8 & 54.86 $\pm$ 8.2 \\
GIN (l=2) - Concat Combine & \textbf{73.78 $\pm$ 6.0} & 69.22 $\pm$ 6.7 & \textbf{60.81 $\pm$ 8.5} \\
\addlinespace
GIN (l=3) - Without Combine & 67.03 $\pm$ 7.2 & 58.82 $\pm$ 6.5 & 43.51 $\pm$ 7.3 \\
GIN (l=3) - Sum Combine & 67.03 $\pm$ 5.0 & 66.08 $\pm$ 5.8 & 50.81 $\pm$ 7.8 \\
GIN (l=3) - Mean Combine & 64.86 $\pm$ 6.6 & 63.92 $\pm$ 6.5 & 52.97 $\pm$ 9.5 \\
GIN (l=3) - Concat Combine & 72.70 $\pm$ 4.6 & \textbf{71.76 $\pm$ 5.1} & 59.73 $\pm$ 11.3 \\
\addlinespace
GIN (l=4) - Without Combine & 67.57 $\pm$ 5.5 & 59.41 $\pm$ 3.8 & 42.70 $\pm$ 5.0 \\
GIN (l=4) - Sum Combine & 66.49 $\pm$ 8.4 & 63.14 $\pm$ 6.2 & 52.16 $\pm$ 10.1 \\
GIN (l=4) - Mean Combine & 70.54 $\pm$ 4.6 & 63.73 $\pm$ 8.9 & 49.73 $\pm$ 9.8 \\
GIN (l=4) - Concat Combine & 69.73 $\pm$ 6.1 & 68.63 $\pm$ 7.9 & 57.30 $\pm$ 7.7 \\
\addlinespace
GIN (l=5) - Without Combine & 62.70 $\pm$ 6.8 & 52.16 $\pm$ 8.0 & 44.86 $\pm$ 7.0 \\
GIN (l=5) - Sum Combine & 67.57 $\pm$ 6.6 & 60.59 $\pm$ 6.4 & 47.03 $\pm$ 8.4 \\
GIN (l=5) - Mean Combine & 70.54 $\pm$ 6.7 & 57.84 $\pm$ 7.7 & 48.11 $\pm$ 10.9 \\
GIN (l=5) - Concat Combine & 73.24 $\pm$ 4.6 & 66.27 $\pm$ 3.7 & 51.89 $\pm$ 7.7 \\
\midrule
\addlinespace
DAGMLP (l=2; $0$-NTs) - Without Combine & 74.59 $\pm$ 4.7 & 65.10 $\pm$ 6.6 & 60.81 $\pm$ 4.6 \\
DAGMLP (l=2; $0$-NTs) - Sum Combine & \textbf{85.68 $\pm$ 4.8} & 77.84 $\pm$ 5.3 & 68.38 $\pm$ 4.5 \\
DAGMLP (l=2; $0$-NTs) - Mean Combine & 77.03 $\pm$ 6.3 & 76.67 $\pm$ 3.7 & 65.95 $\pm$ 5.7 \\
DAGMLP (l=2; $0$-NTs) - Concat Combine & 81.35 $\pm$ 7.1 & \textbf{79.41 $\pm$ 4.6} & \textbf{68.92 $\pm$ 5.4} \\
\addlinespace
DAGMLP (l=3; $0$-NTs) - Without Combine & 74.86 $\pm$ 8.8 & 66.08 $\pm$ 4.3 & 57.84 $\pm$ 2.5 \\
DAGMLP (l=3; $0$-NTs) - Sum Combine & 78.11 $\pm$ 6.6 & 76.67 $\pm$ 4.5 & 64.59 $\pm$ 5.9 \\
DAGMLP (l=3; $0$-NTs) - MEAN Combine & 77.57 $\pm$ 5.1 & 74.31 $\pm$ 5.1 & 63.24 $\pm$ 4.0 \\
DAGMLP (l=3; $0$-NTs) - Concat Combine & 80.27 $\pm$ 8.1 & 78.82 $\pm$ 5.2 & 64.05 $\pm$ 5.8 \\
\midrule
\addlinespace
DAGMLP (l=2; $1$-NTs) - Without Combine & 65.68 $\pm$ 5.4 & 64.51 $\pm$ 7.4 & 55.95 $\pm$ 6.7 \\
DAGMLP (l=2; $1$-NTs) - Sum Combine &\textbf{ 80.54 $\pm$ 6.0} & 79.22 $\pm$ 6.3 & 67.03 $\pm$ 6.7 \\
DAGMLP (l=2; $1$-NTs) - Mean Combine & 78.38 $\pm$ 5.7 & \textbf{79.61 $\pm$ 3.9} & 66.49 $\pm$ 5.0 \\
DAGMLP (l=2; $1$-NTs) - Concat Combine & 79.73 $\pm$ 3.7 & 79.61 $\pm$ 5.1 & \textbf{69.19 $\pm$ 4.6} \\
\addlinespace
DAGMLP (l=3; $1$-NTs) - Without Combine & 65.95 $\pm$ 6.3 & 60.98 $\pm$ 7.9 & 54.59 $\pm$ 6.3 \\
DAGMLP (l=3; $1$-NTs) - Sum Combine & 78.38 $\pm$ 6.2 & 79.61 $\pm$ 5.2 & 64.59 $\pm$ 5.0 \\
DAGMLP (l=3; $1$-NTs) - Mean Combine & 73.51 $\pm$ 5.4 & 75.29 $\pm$ 4.9 & 66.49 $\pm$ 4.9 \\
DAGMLP (l=3; $1$-NTs) - Concat Combine & 80.27 $\pm$ 6.0 & 78.63 $\pm$ 4.8 & 67.03 $\pm$ 2.6 \\
\bottomrule
\end{tabular}
\end{table}

\section{Datasets}
\label{app:datasets}
We provide information about the datasets used in the experimental evaluation. Table~\ref{tab:datasets_summary} provides an overview of the datasets, along with their corresponding characteristics. %

\paragraph{Synthetic datasets.}
(1) EXP-Classification (EXP-Class) and EXP-Isomorphic (EXP-Iso) evaluate GNN expressivity, featuring graph pairs with varying SAT outcomes and 1-WL distinguishability~\citep{abboud2021surprising}. EXP-Class extends EXP-Iso by including 50\% ``corrupted'' data, making the learning task more challenging.
(2) Circulant Skip Links (CSL) graphs~\citep{murphy2019relational} are highly symmetric, 4-regular graphs that consist of a cycle with additional 'skip links.' Despite their symmetry, these graphs present a challenge for the WL test and GNNs based on WL, as these methods fail to distinguish between non-isomorphic instances of such graphs.

\paragraph{Real-world datasets.} We examine Texas, Wisconsin, Cornell, Cora, CiteSeer, PubMed, MUTAG, IMDB-B, IMDB-M, ENZYMES, and PROTEINS from ~\citep{Craven1998LearningTE, McCallum2000AutomatingTC, Giles1998CiteSeerAA, Sen2008CollectiveCI, Datasets}. Texas, Wisconsin, and Cornell are web page datasets, each representing a different university's web domain. Cora dataset comprises scientific publications classified into seven categories, making it a standard benchmark for citation network studies. CiteSeer is another citation network dataset, including academic papers for document classification. The PubMed dataset, derived from biomedical literature, leveraging its rich metadata encompassing abstracts, citations, and other bibliometric information. IMDB-B and IMDB-M are movie network datasets for binary and multi-class classification, respectively. ENZYMES has six protein graph classes, while PROTEINS represents a binary classification task from bioinformatics.

\begin{table}[tb]
\centering
\caption{Summary of characteristics for the synthetic datasets~\citep{murphy2019relational, abboud2021surprising} and TUDatasets~\citep{Datasets}. The table provides information on the dataset name, number of graphs (\boldmath\textbf{$\vert$G$\vert$}), average number of nodes (\boldmath\textbf{$\overline{\vert V\vert}$}), average number of edges (\boldmath\textbf{$\overline{\vert E \vert}$}), and average diameter  (\boldmath\textbf{$\overline{ D}$}) for each dataset.}
\label{tab:datasets_summary}
\begin{tabular}{lcccc}
	\toprule
	\textbf{Dataset} & \boldmath\textbf{$\vert$G$\vert$} & \boldmath\textbf{$\overline{\vert V\vert}$} & \boldmath\textbf{$\overline{\vert E \vert}$} & \boldmath\textbf{$\overline{D }$} \\
	\midrule
	\textbf{CSL} & 150 & 41.0 & 164.0 & 6.0 \\
	\textbf{EXP-Class} & 1200 & 55.8 & 139.6 & 12.6 \\
	\textbf{EXP-Iso} & 1200 & 44.4 & 110.2 & 8.5 \\
	\midrule
	\textbf{MUTAG} & 188 & 17.93 &  39.59 &  8.22 \\
	\textbf{IMDB-B} & 1000 & 19.8 & 193.1 & 1.9 \\
	\textbf{IMDB-M} & 1500 & 13.0 & 131.9 & 1.5 \\
	\textbf{ENZYMES} & 600 & 32.6 & 124.3 & 10.9 \\
	\textbf{PROTEINS} & 1113 & 39.1 & 145.6 & 11.6 \\
	\midrule
	\textbf{Texas} & 1 & 183 & 325 & 8 \\
	\textbf{Wisconsin} & 1 & 251 & 515 & 8 \\
	\textbf{Cornell } & 1 & 183 & 298 & 8 \\
	\textbf{Cora} & 1 & 2708 & 10556 & 19 \\
	\textbf{CiteSeer} & 1 & 3327 & 9104 & 28 \\
	\textbf{PubMed} & 1 & 19717 & 88648 & 18 \\
	\bottomrule
\end{tabular}
\end{table}

\section{Running Time}
The running time for generating and merging $0$- and $1$-NTs with different layers on different datasets is presented in Figure~\ref{fig:dag_build_time_0k}. We employ a parallelized algorithm to construct the NTs, where each graph is also processed in parallel.

\begin{figure}[tb]
\centering
\includegraphics[width=0.99\textwidth]{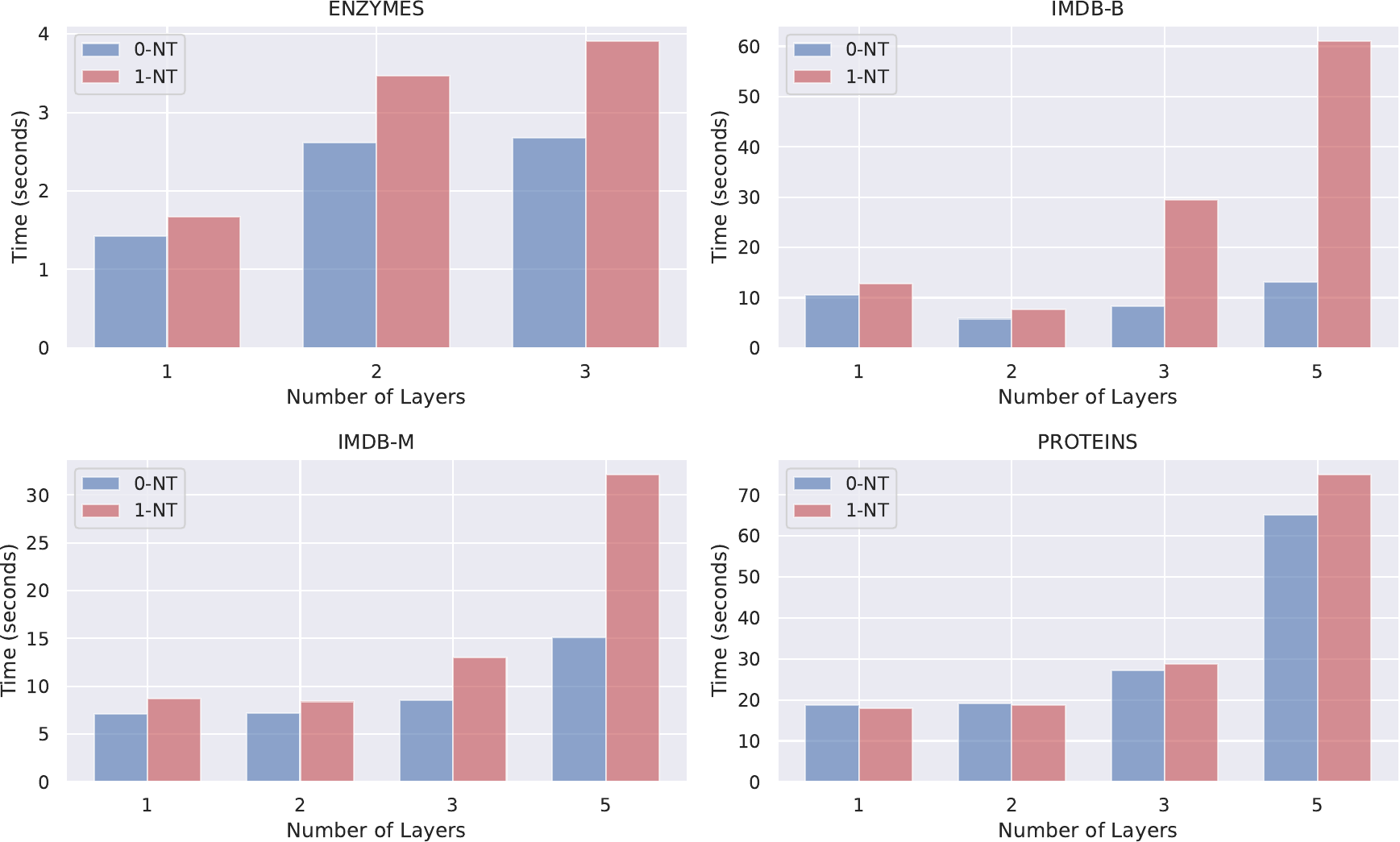}
\caption{Running time for building the $0$- and $1$-redundant NTs.}
\label{fig:dag_build_time_0k}        
\end{figure}

\section{Hyper-Parameters}
\label{sec:hyperparameter}
The hyper-parameters used for the synthetic datasets can be seen in Table~\ref{tab:synthetic_dataset_config}. The hyper-parameters for the TUDataset experiments were chosen as follows:
The batch size for training is set to $64$. Learning rate (LR) is set to $0.001$. The classifier trained for $500$ epochs. The dimension of the embedding is set to $128$. The optimizer used is Adam, and the scheduler is set to \texttt{StepLR} with a step size of $100$ and a gamma value of $0.5$. The aggregation method (pooling) is defined as \texttt{mean} and a dropout rate of $0.5$ is specified. Early stopping is configured with a patience of $250$ epochs and uses accuracy instead of loss. The number of layers for each dataset is set as in Table~\ref{tab:layer_config}, and shuffling of the dataset is enabled.

\begin{table}[tb]
\centering
\caption{Synthetic dataset hyper-parameter configuration details.}
\label{tab:synthetic_dataset_config}
\resizebox{\textwidth}{!}{
\begin{tabular}{lcccccccc}
	\toprule
	\textbf{Dataset} & \textbf{Task} & \textbf{Embedding} & \textbf{Target} & \textbf{Layers} & \textbf{Batch Size} & \textbf{Epochs} & \textbf{LR} \\
	\midrule
	\textbf{EXP-Class} & Classification & 64 & 10 & 15 & 32 & 200 & $10^{-3}$ \\
	\textbf{EXP-Iso} & Isomorphism Test & 1 & 1 & 6 & 1 & - & - \\
	\textbf{CSL} & Classification & 64 & 10 & 6 & 32 & 200 & $10^{-3}$ \\
	\bottomrule
\end{tabular}
}
\end{table}

\begin{table}[tb]
\centering
\caption{TUDatasets layer configuration details.}
\label{tab:layer_config}
\begin{tabular}{lccccc}
	\toprule
	\textbf{Dataset} & \textbf{IMDB-B} & \textbf{IMDB-M} & \textbf{ENZYMES} & \textbf{PROTEINS} \\
	\midrule
	\textbf{Layers} & 5, 3, 2 & 5, 3, 2 & 3, 2 & 5, 3, 2 \\
	\bottomrule
\end{tabular}
\end{table}

\section{Hardware}
\label{sec:hardware}
The hardware configuration consists of dual AMD 7252 CPUs, each with 8 cores, and two NVIDIA A40 GPUs. The system is supplemented with 256 GB of RAM. Each NVIDIA A40 GPU comes with 10,752 CUDA cores and a clock frequency of 1.305 GHz. The GPUs have 48 GB of memory and a bandwidth of 696 GB/s, operating at a Thermal Design Power (TDP) of 300 Watts. In terms of performance, the GPUs can deliver 37,400 GFLOPs in single-precision (FP32) and 1,169 GFLOPs in double-precision (FP64) computations.

\end{document}